\documentclass{article}
\usepackage{todonotes}

\usepackage{ijcai24}

\usepackage{times}
\usepackage{soul}
\usepackage{url}
\usepackage[hidelinks]{hyperref}
\usepackage[utf8]{inputenc}
\usepackage[small]{caption}
\usepackage{graphicx}
\usepackage{amsmath}
\usepackage{amsthm}
\usepackage{booktabs}
\usepackage{algorithm}
\usepackage{algorithmic}
\usepackage[switch]{lineno}
\usepackage[final]{microtype}

\newtheorem{example}{Example}
\newtheorem{theorem}{Theorem}
\newtheorem*{theorem*}{Theorem}

\usepackage{ellipsis, ragged2e}
\usepackage{amsfonts}
\usepackage{mdframed}
\usepackage[capitalize]{cleveref}
\usepackage{paralist}
\usepackage{comment}
\usepackage{xcolor}
\usepackage[T1]{fontenc}

\newtheorem{definition}{Definition}

\newtheorem{lemma}{Lemma}
\newtheorem{proposition}{Proposition}

\input{0_macros}

\advance\textwidth2mm 
\advance\hoffset-1mm
\advance\textheight3mm
\advance\voffset-1.5mm
\setdefaultleftmargin{1em}{1em}{}{}{}{}

\definecolor{Green}{HTML}{00A64F}
\newcommand{\boxlabel}[1]{\sffamily\Huge #1}
\tikzset{
	thick head/.style={>={stealth[width=6pt,length=8pt]}},
	thick path/.style={every path/.style={thick}},
	every picture/.style={
		>=stealth,thick path
	},
	gridbox/.style={draw=black,rectangle,minimum width=1cm,minimum height=1cm,font=\sffamily,fill=gray!10},
	slippery/.style={gridbox,font={\boxlabel{S}}},
	goal/.style={gridbox,fill=Green,font={\boxlabel{G}}},
	init/.style={gridbox,font={\boxlabel{I}}},
	limited/.style={gridbox,draw=red,fill=orange!40},
	obstacle/.style={gridbox,fill=red,font={\boxlabel{X}}},
	boxO/.style={gridbox},
	boxX/.style={obstacle},
	boxs/.style={slippery,font={\boxlabel{s}}},
	boxS/.style={slippery},
	boxL/.style={limited},
	boxt/.style={limited,font={\boxlabel{s}}},
	boxT/.style={limited,font={\boxlabel{S}}},
	boxG/.style={goal},
	boxI/.style={init},
	boxR/.style={gridbox,font={\boxlabel{$\rightarrow$}}},
	boxD/.style={gridbox,font={\boxlabel{$\downarrow$}}},
	boxU/.style={gridbox,font={\boxlabel{$\uparrow$}}},
	boxr/.style={limited,font={\boxlabel{$\rightarrow$}}},
	boxF/.style={gridbox,fill=orange,font={\boxlabel{F}}},
}

\title{Certified Policy Verification and Synthesis for MDPs under\\
	Distributional Reach-avoidance Properties}

\author{
	S.~Akshay$^1$\and
	Krishnendu Chatterjee$^2$\and
	Tobias Meggendorfer$^3$\and
	\DJ or\dj e \v{Z}ikeli\'c$^4$
	\affiliations
	$^1$Indian Institute of Technology Bombay, India\\
	$^2$Institute of Science and Technology Austria (ISTA), Austria\\
	$^3$Lancaster University Leipzig, Germany\\
	$^4$Singapore Management University, Singapore
	\emails
	akshayss@cse.iitb.ac.in, krishnendu.chatterjee@ist.ac.at, tobias@meggendorfer.de, dzikelic@smu.edu.sg
}

\begin{document}
\maketitle
\begin{abstract}
Markov Decision Processes (MDPs) are a classical model for decision making in the presence of uncertainty. Often they are viewed as state transformers with planning objectives defined with respect to paths over MDP states. An increasingly popular alternative is to view them as distribution transformers, giving rise to a sequence of probability distributions over MDP states. For instance, reachability and safety properties in modeling robot swarms or chemical reaction networks are naturally defined in terms of probability distributions over states. Verifying such distributional properties is known to be hard and often beyond the reach of classical state-based verification techniques.

In this work, we consider the problems of {\em certified} policy (i.e.~controller) verification and synthesis in MDPs under distributional reach-avoidance specifications. By certified we mean that, along with a policy, we also aim to synthesize a (checkable) certificate ensuring that the MDP indeed satisfies the property. Thus, given the target set of distributions and an unsafe set of distributions over MDP states, our goal is to either synthesize a certificate for a given policy or synthesize a policy along with a certificate, proving that the target distribution can be reached while avoiding unsafe distributions. To solve this problem, we introduce the novel notion of distributional reach-avoid certificates and present automated procedures for (1)~synthesizing a certificate for a given policy, and (2)~synthesizing a policy together with the certificate, both providing {\em formal guarantees} on certificate correctness. Our experimental evaluation demonstrates the ability of our method to solve several non-trivial examples, including a multi-agent robot-swarm model, to synthesize certified policies and to certify existing policies.

\end{abstract}

	

\section{Introduction}

\noindent{\bf State transformer view of MDPs.} Markov decision processes (MDPs) are a classical model for decision making in the presence of uncertainty. The prevalent view of MDPs defines them as {\em state transformers}. Under a policy that resolves non-determinism, an MDP defines a purely stochastic system that performs 
probabilistic moves from a state to another. 
This gives rise to a probability space over the set of all runs, i.e.~infinite sequences of states, in the MDP~\cite{DBLP:books/daglib/0020348}. MDPs are a central object of study within the AI, planning, and formal methods communities.
There is a rich body of work on scalable techniques for reasoning about various properties in MDPs such as discounted-sum and long-run average reward on one hand~\cite{DBLP:books/wi/Puterman94} and the computation of precise probabilities with which a qualitative objective is satisfied on the other hand, including model checking over expressive logics such as PCTL*~\cite{DBLP:conf/sfm/KwiatkowskaNP07}.

\smallskip\noindent{\bf Distribution transformer view of MDPs.}  While there is a lot of literature on analyzing MDPs as state transformers, there are several application domains where these approaches fall short. For instance, consider a path planning problem for a robot swarm consisting of an arbitrary number of robots distributed over a set of states. The states are arranged in a topology that has obstacles that must be avoided and a target set of states that must be reached. We want that at least 90\% of the robots must reach the target eventually, but at any intermediate step less than 10\% must be stuck in an obstacle. In other words, we want to synthesize a policy to control the robot swarm such that a {\em distributional reach-avoid property} defined with respect to distributions of robots is satisfied. To do this under state-based view, we would need to take the product of state spaces for each agent and define policies over this product space, which would be highly inefficient for systems with many agents.

An alternative to the state-based view is to view MDPs as {\em distribution transformers}. In this view, starting from some initial probability distribution over MDP states and under a policy that resolves non-determinism, the MDP at each time step induces a new distribution over states, giving rise to a sequence of distributions. Reasoning about this sequence provides a much more natural framework for controlling multi-agent systems under specifications defined in terms of positions of each agent. This allows one to define {\em distributional policies} that prescribe actions to be performed by each agent based on the current distribution of positions of all other agents, thus providing an effective and compact way for agents to “communicate” their positions to each other. Hence, in contrast to the state-based view, distributional policies are more convenient for controlling multi-agent systems with large number of agents where it suffices to only reason about distributions of their positions and not about positions of each individual agent. 
In addition to robot-swarms~\cite{BaldoniBMR08}, the distribution transformer view of MDPs also naturally arises in other applications such as bio-chemical reaction networks~\cite{DBLP:conf/qest/KorthikantiVAK10,HenzingerMW09} where populations/concentration of cells are distributed across states.

\smallskip\noindent{\bf Certification of policies.} In this work, we consider the problem of {\em automated and formal certification of correctness} of policies with respect to distributional reach-avoidance properties in MDPs. This is important in safety-critical applications including robot-swarms or bio-chemical reaction networks, where it is imperative to provide guarantees on correctness prior to policy deployment. In order to ensure safe and correct behavior of such systems, we are interested in computing policies together with {\em certificates} that serve as formal proofs of correctness of policies and allow for safe and trustworthy policy deployment. We ask the following research questions:
\begin{compactenum}
	\item {\em Certificates for policies.} What should be a certificate for formally reasoning about distributional reach-avoidance properties in MDPs? A good certificate should be an object that simultaneously allows {\em formal} and {\em automated} reasoning about its correctness.
	\item {\em Formal policy verification with certificates.} Given an MDP and a policy, how do we {\em compute} such a certificate that formally proves correctness of the policy?
	\item {\em Formal policy synthesis with certificates.} Given an MDP, how do we compute a policy together with a certificate that formally proves its correctness? Can we synthesize not only memoryless but also distributional policies?
\end{compactenum}

\smallskip\noindent{\bf Prior work and challenges.} Recent years have seen increased interest in formal analysis of MDPs under the distributional view. It was shown~\cite{DBLP:journals/ipl/AkshayAOW15} that the problem of deciding whether a policy is correct with respect to distributional reachability (and hence distributional reach-avoidance) properties is extremely hard; in fact as hard as the so-called Skolem problem, a long-standing number-theoretic problem whose decidability is unknown~\cite{Lipton22,DBLP:conf/rp/OuaknineW12}. Moreover, it was shown in~\cite{DBLP:journals/logcom/BeauquierRS06} that distributional properties such as reachability and safety cannot be expressed in PCTL*, hence classical model checking methods are not applicable to them. Over the years, the verification community has often studied MDPs under the distributional view, however existing works are either theoretical in nature and focus on decidability of the problem or its variants for different subclasses of MDPs~\cite{DBLP:journals/tse/KwonA11,DBLP:journals/jacm/AgrawalAGT15,DBLP:journals/logcom/BeauquierRS06,0001MS14,DBLP:conf/lics/AkshayGV18} or study specialized logics for reasoning about distributional properties~\cite{DBLP:journals/jacm/AgrawalAGT15,DBLP:journals/logcom/BeauquierRS06}. Existing automated methods are restricted to distributional safety~\cite{AkshayCMZ23}.

To the best of our knowledge, there exists no prior automated method for formal policy verification or synthesis in MDPs with respect to distributional reachability or reach-avoidance properties. Given the Skolem-hardness of the problem, a natural question to ask is how to address this problem in a way which provides formal correctness guarantees while at the same time being practically applicable. Motivated by the success of termination and safety analysis in program verification, we consider an over-approximative approach which may not terminate in all cases but which works in practice while preserving formal guarantees on the correctness of its outputs.

\smallskip\noindent{\bf Contributions.} Our contributions are as follows:
\begin{compactenum}
  \item {\em Certificate for distributional reach-avoidance.} We introduce the novel notion of {\em distributional reach-avoid certificates}, and show that they provide a sound and complete proof rule for distributional reach-avoidance  (Section~\ref{sec:certificates}).
  \item {\em Algorithms for formal verification and synthesis.} We develop novel {\em template-based synthesis algorithms} for the formal synthesis and verification problems with respect to distributional reach-avoidance properties in MDPs. 
    \begin{compactenum}
  	\item First, we develop an algorithm for synthesizing {\em memoryless policies} along with {\em affine} distributional reach-avoid certificates. Memoryless policies can be efficiently deployed and executed and are thus preferred in practice. The algorithm is {\em sound and relatively complete} for deciding the existence of and for computing a memoryless policy and an affine distributional reach-avoid certificate, whenever they exist. While our notion of distributional reach-avoid certificates in Section~\ref{sec:certificates} applies in the general case and provides a sound and complete proof rule for distributional reach-avoidance, our algorithm focuses on the family of {\em affine} distributional reach-avoid certificates for {\em practical reasons}, in order to allow for their fully automated and efficient computation. 
  	(Section~\ref{sec:algomemoryless}).
  	\item While memoryless policies are preferred in practice, they are not always sufficient for solving distributional reach-avoid tasks and one may even require unbounded memory. To that end, we show that it suffices to restrict to the so-called {\em distributionally memoryless policies} (Section~\ref{sec:certificates}) and develop an algorithm for synthesizing them together with affine distributional reach-avoid certificates. The algorithm is sound but incomplete (Section~\ref{sec:algodistributional}).
  	 \item Finally, in both cases, we also develop a certification algorithm that proves the correctness of a {\em given policy} by computing an affine distributional reach-avoidance certificate for it (Sections~5 and~6).
  \end{compactenum}
  \item {\em Experimental evaluation.} We implement a prototype of our approach and show that it is able to solve several distributional reach-avoid tasks, including robot-swarms in gridworld environments. Our prototype tool achieves impressive results even when restricted to memoryless strategies, thus showing the effectiveness of our approach as well as the generality of the relative completeness guarantees provided by our first algorithm (Section~\ref{sec:experiments}).
\end{compactenum}

\smallskip\noindent{\bf Related work.}
Unlike in the distributional case discussed above, probabilistic reach-avoidance over MDP states is solvable in polynomial time~\cite{DBLP:books/daglib/0020348} and formal policy synthesis for state properties in finite MDPs has been extensively studied. In addition, recent years have seen increased interest in formal policy synthesis~\cite{SoudjaniGA15,LavaeiKSZ20,cauchi2019stochy,BadingsA00PS22,BadingsRA023,XLZF21,ZikelicLHC23,ZikelicLVCH23,DBLP:conf/concur/GroverKMW22} and certification of policies~\cite{AlshiekhBEKNT18,0001KJSB20,LechnerZCH22} for continuous-state MDPs. However, none of these methods are applicable to the {\em distributional} reach-avoidance problem.

Of works considering distributional properties, the most closely related is the recent work of~\cite{AkshayCMZ23} which considers distributional safety and also proposes a template-based synthesis method. However, our work differs in three important ways. First, our method supports distributional reachability and reach-avoidance. Formal analysis of distributional reachability and reach-avoidance, and even the very definition of a certificate, is significantly more involved and as a result the proofs of our Theorems~\ref{stm:dist_strat_is_enough} and~\ref{thm:soundandcomplete} are more challenging than the distributional safety setting. Second, we consider both {\em universal and existential} distributional problems (see Section~\ref{sec:problem} for definition), whereas they only considered the existential case, i.e.\ for a single initial distribution. Third, our automated method allows target and safe sets to be specified {\em both} in terms of strict and non-strict inequalities. This is one of the highlighted open problems in~\cite[Section~8]{AkshayCMZ23}. 

The template-based synthesis approach has also been extensively used for controller synthesis for state properties in deterministic~\cite{jarvis2003some,AhmadiM16} and stochastic systems~\cite{PrajnaJP07}, as well as in program analysis~\cite{GulwaniSV08}. In particular, our distributional reach-avoid certificate draws insights from ranking functions for termination~\cite{ColonS01} and invariants for safety analysis in programs~\cite{ColonSS03}.










\section{Preliminaries}
\label{sec:prelim}
%
A \emph{probability distribution} on a (countable) set $X$ is a mapping $\distribution : X \to [0,1]$, such that $\sum_{x\in X} \distribution(x) = 1$. We write $\support(\distribution) = \{x \in X \mid \distribution(x) > 0\}$ to denote its \emph{support}, and $\Distributions(X)$ to denote the set of all probability distributions on~$X$.

\smallskip\noindent{\bf Markov decision processes.}
%
A \emph{Markov decision process (MDP)} is a tuple $\MDP = (\States, \Actions, \transitions)$, where
	$\States$ is a finite set of \emph{states},
	$\Actions$ is a finite set of \emph{actions}, overloaded to yield for each state $s$ the set of \emph{available actions} $\Actions(s) \subseteq \Actions$, and
	$\transitions: \States \times \Actions \to \Distributions(\States)$ is a \emph{transition function} that for each state $s$ and (available) action $a \in \Actions(s)$ yields a probability distribution over successor states.
A {\em Markov chain} is an MDP where each state only has a single available action.

An infinite path in an MDP is a sequence $\infinitepath = s_1 a_1 s_2 a_2 \cdots \in (\States \times \Actions)^\omega$ such that $a_i \in \Actions(s_i)$ and $\transitions(s_i,a_i, s_{i+1}) > 0$ for every $i \in \Naturals$. A {\em finite path} is a finite prefix of an infinite path.
We use $\infinitepath_i$ and $\finitepath_i$ to refer to the $i$-th state in the given (in)finite path, and $\IPaths<M>$ and $\FPaths<M>$ for the set of all (in)finite paths of $M$.

Dynamics of MDPs are defined in terms of policies. A {\em policy} in an MDP is a map $\strategy : \FPaths<\MDP> \to \Distributions(\Actions)$, which given a finite path $\finitepath = s_0 a_0 s_1 a_1 \dots s_n$ yields a probability distribution $\strategy(\finitepath) \in \Distributions(\Actions(s_n))$ on the actions to be taken next. A policy is {\em memoryless} if the probability distribution over actions only depends on the current state and not on the whole history, i.e.~if $\strategy(\finitepath) = \strategy(\finitepath')$ whenever $\finitepath$ and $\finitepath'$ end in the same state.
Fixing a policy $\strategy$ and initial distribution $\distribution_0$ induces 
a unique probability measure $\ProbabilityMC<\MDP^\strategy, \distribution_0>$ over infinite paths of $\MDP$~\cite{DBLP:books/wi/Puterman94}.



\smallskip\noindent{\bf MDPs as distribution transformers.}
MDPs are traditionally viewed as {\em random generators} of paths, and one investigates the (expected) behaviour of a generated path, i.e.\ path properties.
However, in this work we treat probabilistic systems as \emph{(deterministic) transformers of distributions}.

First, fix a Markov chain $\MC$.
For a given initial distribution $\distribution_0$, we define the distribution at step $i$ by $\distribution_i(s) = \ProbabilityMC<\distribution_0>[\{\infinitepath \in \IPaths<\MC> \mid \infinitepath_i = s\}]$, i.e.\ the probability to be in state $s$ at step $i$.
We write $\distribution_i = \MC(\distribution_0, i)$ for the $i$-th distribution and $\distribution_1 = \MC(\distribution_0)$ for the {\em one-step} application of this transformation.
Likewise, we obtain the same notion for an MDP $\MDP$ combined with a policy $\strategy$, and write $\distribution_i = \MDP^\strategy(\distribution_0, i)$, $\distribution_1 = \MDP^\strategy(\distribution_0)$.
In summary, for a given initial distribution, a Markov chain induces a unique stream of distributions, and an MDP provides one for each policy.
This naturally invites questions related to this induced stream of distributions.
In their path interpretation, queries on MDPs such as \emph{reachability} or \emph{safety}, i.e.\ asking the probability of reaching or avoiding a set of states, allow for simple, polynomial time solutions \cite{DBLP:books/wi/Puterman94,DBLP:books/daglib/0020348}.
However, the corresponding problems in the space of distributions are surprisingly difficult. Our goal is to enable efficient and fully automated reachability and safety analyses under the distribution transformer interpretation of Markov chains and MDPs. 
We start with an example.

\begin{figure}[t]
\begin{center}
	\begin{tikzpicture}
		\foreach [count=\i from 0,evaluate=\i as \x using {mod(\i,4)},evaluate=\i as \y using -floor(\i / 4)] \l in {
			init,obstacle,gridbox,goal,%
			slippery,limited,slippery,gridbox%
		}
			\node[\l] at (\x,\y) {};
	\end{tikzpicture}
	\caption{Gridworld example.}
	\label{fig:grid-running}
\end{center}
\end{figure}
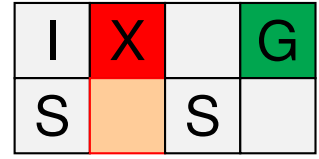

\begin{example}
  \label{eg:grid-running}
  Consider a $2\times 4$ grid as depicted in Figure~\ref{fig:grid-running}. Let us say that a robot swarm starts at the initial cell marked $I$ in the top-left of the grid. From any cell any robot can move horizontally or vertically to an adjacent cell via non-deterministic moves, as long as the adjoining cell is not marked $X$. Cells marked $X$ are obstacles that cannot be moved to. Cells marked $S$ are stochastic where $10\%$ of the robots remain in the cell while remaining $90\%$ can move to adjoining cell. Further, orange cells are distributional obstacles, i.e.\ at any point only $10\%$ of all robots in the swarm may be in the set of orange cells. One could visualize this as a narrow path that ensures that only few robots can go through it safely at any point. Finally, there is a goal cell $G$. The problem is then to go from $I$ to $G$ (at least $90\%$ of the swarm must reach $G$), while the dynamics must follow the stochastic constraints in $S$ cells, and at the same time avoid obstacles $X$ and satisfy distributional constraints in orange cells.
\end{example}

\section{Distributional Reach-avoidance Problems}\label{sec:problem}
Consider a \emph{target set} of distributions $T$ and a \emph{safe set} of distributions $H$, i.e.~complement of the set of unsafe distributions. We consider a safe set rather than its complement for the simplicity of the presentation. Distributional reach-avoidance is concerned with verifying or computing an MDP policy under which the induced stream of distributions stays in $H$ until it reaches $T$. There are several natural variants of the problem, depending on whether policy and initial distribution are given. 

Consider an MDP $\MDP = (\States, \Actions, \transitions)$, a set of initial distributions $\Init \subseteq \Distributions(\States)$,$T \subseteq \Distributions(\States)$, and $H \subseteq \Distributions(\States)$. Let $\pi$ be a (memoryless/general) policy for $\MDP$. Then:
\begin{compactitem}
	\item For $\distribution_0\in \Init$, we say that the MDP $\MDP$ satisfies \emph{$(T,H)$-reach-avoidance from $\distribution_0$ under policy $\strategy$}, if there exists $i \geq 0$ such that $\MDP^\strategy(\distribution_0, i) \in T$ and if for all $0 \leq j < i$ we have $\MDP^\strategy(\distribution_0, j) \in H$. When $\mu_0$ is fixed, we also say that $\MDP$ satisfies {\em unit-$(T,H)$-reach-avoidance under $\pi$}.
	\item We say $\MDP$ satisfies \textit{existential-$(T,H)$-reach-avoidance under $\strategy$} if there exists a distribution $\distribution_0 \in \Init$ such that $\MDP$ satisfies $(T,H)$-reach-avoidance from $\distribution_0$ under $\strategy$. 
	\item We say MDP $\MDP$ satisfies \textit{universal-$(T,H)$-reach -avoidance under $\strategy$} if for all $\distribution_0 \in \Init$, $\MDP$ satisfies $(T,U)$-reach-avoidance from $\distribution_0$ under $\strategy$.
	\end{compactitem}
Existential and universal reach-avoidance are relevant when the input distribution is not precise. While universal reach-avoidance is more restrictive than existential, it is motivated by robustness questions, where we do not precisely know the initial distribution so we ask that all distributions belonging to some uncertainty set behave in a certain way. We are now ready to formally define the problems that we consider:
\begin{mdframed}
	\begin{center}
		\textbf{Distributional Reach-avoidance}
	\end{center}%
	\noindent Given an MDP $\MDP = (\States, \Actions, \transitions)$, a set of initial distributions $\Init \subseteq \Distributions(\States)$, $\mu_0\in \Init$, a target set $T \subseteq \Distributions(\States)$, and a safe set $H \subseteq \Distributions(\States)$:
	
	\indent (i) \textbf{The (unit/existential/universal) Policy Certification problem} asks if, given a policy $\pi$, there exists a certificate $\certificate$ which ensures that $\MDP$ satisfies (unit/ex-istential/universal) $(T,H)$-reach-avoidance under $\pi$ and allows automated reasoning about certificate correctness. We refer to the pair $(\pi,\certificate)$ as a {\em certified policy}.
	
	\indent (ii) \textbf{The (unit/existential/universal) Policy Verification problem} asks to, given a policy $\pi$, compute a certificate $\certificate$ such that $(\pi,\certificate)$ as a certified policy.
	
	\indent (iii) \textbf{The (unit/existential/universal) Policy Synthesis problem} asks to compute a policy $\pi$ {\em together} with a certificate $\certificate$ such that $(\pi,\certificate)$ as a certified policy.
\end{mdframed}        

We note that these problems are not yet fully well defined: Since $\Distributions(\States)$ is an uncountably infinite space, we need to describe how the sets $\Init$, $T$, and $H$ are represented.
As common for continuous scenarios, we restrict to \emph{affine} sets.
Concretely, we assume that these sets are specified via a finite number of affine inequalities and are of the form
\begin{align*}
	T     & = \Big\{\distribution \in \Distributions(\States) \mid {\bigwedge}_{j=1}^{N_T} (t_0^j + {\sum}_{i=1}^n t_i^j \cdot \distribution(s_i)) \bowtie 0 \Big\},       \\
	H     & = \Big\{\distribution \in \Distributions(\States) \mid {\bigwedge}_{j=1}^{N_U} (h_0^j + {\sum}_{i=1}^n h_i^j \cdot \distribution(s_i)) \bowtie 0 \Big\},       \\
	\Init & = \Big\{\distribution \in \Distributions(\States) \mid {\bigwedge}_{j=1}^{N_{\Init}} (a_0^j + {\sum}_{i=1}^n a_i^j \cdot \distribution(s_i)) \bowtie 0 \Big\},
\end{align*}
where $S=\{s_1,\dots,s_n\}$ are MDP states, $t^i_j$, $h_i^j$, and $a_i^j$ are real-valued constants, ${\bowtie} \in \{\geq,>\}$, and $N_T$, $N_H$, $N_{\Init}$ are resp.~numbers of affine inequalities defining $T$, $H$, and $\Init$.

\smallskip\noindent{\bf Hardness.} Note that Policy Verification and Synthesis problems ask only to compute a certificate for a given policy (or policy and certificate), and not to decide their existence. The reason is that, as it turns out, both decision problems are computationally hard even in the setting of affine sets. This hardness emerges already in the case of (memoryless) policy verification for distributional reachability: Even if $M$ is a Markov chain, $|\Init| = 1$, $T = \{\distribution \mid \distribution(s_1) = \frac{1}{4}\}$, and $H = \Distributions(\States)$, the problem is  \textsc{Skolem}-hard\footnote{Intuitively, the \textsc{Skolem} problem asks for a given rational (or integer or real) matrix $M$, whether there exists $n\in\mathbb{N}$, such that $(M^n)_{1,1}=0$  \cite{DBLP:conf/rp/OuaknineW12}. This problem (and its variants) has been the subject of intense research over the last 40 years, see e.g.\ \cite{Lipton22}. Yet, quite surprisingly, it still remains open, even for matrices of size 5 and above.}~\cite{DBLP:journals/ipl/AkshayAOW15}. 
\begin{proposition}\label{prop:hardness}
	The decision problem variants of the unit/existential/universal Policy Verification and Synthesis problems for MDPs with respect to distributional reach-avoidance are as hard as the Skolem problem.
\end{proposition}
As a result, we cannot expect to obtain an efficient, sound and complete decision procedure for our problem. To overcome this, we focus on asking for certificates and policies of a certain {\em special} form, and come up with a sound and \emph{relatively} complete procedure to synthesize them, as explained below.

\section{Proving Distributional Reach-avoidance}\label{sec:certificates}
We now consider the Policy Certification problem discussed in Section~\ref{sec:problem}. First, we show that in order to reason about distributional reach-avoidance, it suffices to restrict to the so-called {\em distributionally memoryless policies}. Second, we introduce our novel certificate for formally proving distributional reach-avoidance, which we call {\em distributional reach-avoid certificate}. We show that distributional reach-avoid certificates provide a sound and complete proof rule for proving distributional reach-avoidance under distributionally memoryless policies. 

\smallskip\noindent{\bf Distributionally memoryless policies.}
Let $\MDP = (\States, \Actions, \transitions)$ be an MDP, $\strategy: \FPaths<\MDP> \to \Distributions(\Actions)$ a policy and $\mu_0 \in \Distributions(\States)$ an initial distribution. Let $\distribution_0,\distribution_1,\dots$ be the stream of distributions induced by $\strategy$ from $\mu_0$. We say that $\strategy$ is {\em distributionally memoryless}, if for any initial distribution $\mu_0 \in \Distributions(\States)$ and for any two finite paths $\finitepath = s_0 a_0 s_1 a_1 \dots s_n$ and $\finitepath' = s_0' a_0' s_1' a_1' \dots s_m'$ with $\distribution_n = \distribution_m$, we have $\strategy(\finitepath) = \strategy(\finitepath')$. Thus, probability distribution over actions only depends on the current distribution over states and not on the whole history. The following theorem shows that, in order to reason about distributional reach-avoidance, it suffices to restrict to distributionally memoryless policies. The proof of the following theorem can be found in Appendix~\ref{app:distmem}.

\begin{theorem} \label{stm:dist_strat_is_enough}
	Let $T,H \subseteq \Distributions(\States)$ be target and safe sets. MDP $\MDP$ satisfies unit/existential/universal-$(T,H)$-reach-avoidance under some policy if and only if there exists a distributionally memoryless policy $\strategy$ such that $\MDP$ satisfies satisfies unit/existential/universal-$(T,H)$-reach-avoidance under $\strategy$.
\end{theorem}

\noindent{\bf Sound and complete certificate.} Intuitively, given a target set of distributions $T$ and a safe set of distributions $H$, a distributional $(T,H)$-reach-avoid certificate under policy $\strategy$ is a pair $\certificate = (R, I)$ of a {\em distributional ranking function} $R$ and a {\em distributional invariant $I$}. The distributional invariant $I$
is a set of distributions that is required to contain all distributions that are reachable under policy $\strategy$ and also to be contained in $H$, while the distributional ranking function $R$ is a function that maps distributions over MDP states to reals, which is required to be nonnegative at all distributions contained in $I$ and to decrease by at least~$1$ after every one-step evolution of the MDP until the target $T$ is reached. We formalize this intuition.

\begin{definition}[Distributional reach-avoid certificate]\label{def:certificate}
  Let $\MDP$ be an MDP, $\distribution_0 \in \Init$ be an initial distribution and $\strategy$ be a policy.  A {\em distributional $(T,H)$-reach-avoid certificate} for $\MDP$ from $\distribution_0$ under $\strategy$ is a pair $(R,I)$, comprising of a function $R: \Distributions(\States) \rightarrow \mathbb{R}$ and a set of distributions $I \subseteq \Distributions(\States)$ such that the following conditions hold:
	\begin{compactenum}
	\item {\em Initial distribution in $I$.} We have $\distribution_0 \in I$.
		\item {\em Inductiveness of $I$ until $T$}. The set $I$ is {\em closed} under application of $\MDP^\strategy$ to any non-target distribution contained in $I$, i.e.~$\MDP^{\strategy}(\distribution) \in I$ holds for every $\distribution \in I \backslash T$.
		\item {\em Safety.} $I$ is a subset of the safe set $H$, i.e.~$I \subseteq H$.
		\item {\em Nonnegativity of $R$.} For every $\distribution \in I$, we have $R(\distribution) \geq 0$.
		\item {\em Strict decrease of $R$ until $T$.} For every $\distribution \in I \backslash T$, we have $R(\distribution) \geq R(\MDP^{\strategy}(\distribution)) + 1$.
	\end{compactenum}
        Furthermore, a distributional $(T,U)$-reach-avoid certificate for $\MDP$ under $\strategy$ is said to be {\em universal}, if it satisfies conditions~2-5 and in addition condition~1 is strengthened to $\Init \subseteq I$, i.e.~the set $I$ must contain all distributions in $\Init$.
\end{definition}
The following theorem establishes that distributional reach-avoid certificates provide a sound and complete proof rule for proving distributional reach-avoidance in MDPs under distributionally memoryless policies. The intuition behind the proof is as follows. Take the distribution transformer view of MDPs and consider the stream $\mu_0,\mu_1,\mu_2,\dots$ of distributions over MDP states induced by starting in $\mu_0$ and repeatedly applying policy $\pi$. Then conditions 1-3 in Definition~\ref{def:certificate} together ensure that distributions in the stream stay in $H$ at least until $T$ is reached. On the other hand, conditions $4$ and $5$ in Definition~\ref{def:certificate} together ensure that a distribution in $T$ must be eventually reached since $R$ cannot be decreased by $1$ indefinitely while remaining non-negative. Hence, $T$ is eventually reached while $H$ is not left in the process, and distributional reach-avoid certificates provide a sound proof rule. To prove completeness, we simply let $I = \{\mu_0,\mu_1,\mu_2,\dots\}$ be the stream of induced distributions, $k$ be the smallest index such that $\mu_k \in T$ and define $R(\mu_i) = \max\{0, k-i\}$. One can then verify that $(R,I)$ is indeed a correct distributional reach-avoid certificate. We defer the formal proof to Appendix~\ref{app:certificate}.

\begin{theorem}[Sound and complete certificates] \label{thm:soundandcomplete}
  Let $\MDP$ be an MDP, $\distribution_0 \in \Init$ and $\strategy$ be a distributionally memoryless policy. Then $\MDP$ satisfies
  \begin{compactenum}
  \item unit-$(T,H)$-reach-avoidance under $\strategy$ iff there exists a distributional $(T,H)$-reach-avoid certificate for $\MDP$ from $\distribution_0$ under $\strategy$.
  \item existential-$(T,H)$-reach-avoidance under $\strategy$ iff there exists a $\mu_0\in\Init$ and distributional $(T,H)$-reach-avoid certificate for $\MDP$ from $\distribution_0$ under $\strategy$.
  \item universal-$(T,H)$-reach-avoidance under $\strategy$ iff there exists a universal distributional $(T,H)$-reach-avoid certificate for $\MDP$ under $\strategy$.
    \end{compactenum}
\end{theorem}
From Proposition~\ref{prop:hardness}, it follows that giving a complete procedure for synthesizing (or indeed even checking existence of) distributional certificates is Skolem-hard. Hence in what follows, we provide an automated template-based overapproximation approach that exploits the advances in SMT-solvers to give as an implementable procedure. First, we restrict to memoryless policies in Section~\ref{sec:algomemoryless}, then we address the general case of distributionally memoryless policies in Section~\ref{sec:algodistributional}.

\section{Algorithm for Memoryless Policies}\label{sec:algomemoryless}

We now consider the (unit/existential/universal) Policy Synthesis and Verification problems defined in Section~\ref{sec:problem} under memoryless policies and present our algorithms for solving these problems. Due to space restrictions, in what follows we directly present our algorithm for solving the existential Policy Synthesis problem under memoryless policies. We then explain how this algorithm can be straightforwardly extended to solve the other problems under the memoryless restriction.

Our algorithm simultaneously synthesizes an initial distribution (for the existential problem), a memoryless policy and an {\em affine} distributional reach-avoid certificate. Restricting to affine distributional reach-avoid certificates (formalized below) ensures efficient and automated computation. While we cannot provide completeness guarantees due to this restriction, we show that the algorithm is sound and {\em relatively complete}, i.e.\ it is guaranteed to compute a memoryless policy and an affine distributional reach-avoid certificate when they exist. 

A distributional reach-avoid certificate $\certificate = (R,I)$ is {\em affine} if it can be specified via affine expressions and inequalities over the distribution space $\Distributions(S)$. That is, the distributional ranking function $R$ is of the form $R  = r_0 + \sum_{i=1}^n r_i \cdot \mu(s_i)$
and the distributional invariant $I$ is of the form
\[ I = \Big\{\distribution \in \Distributions(\States) \mid {\bigwedge}_{j=1}^{N_I} (b_0^j + {\sum}_{i=1}^n b_i^j \cdot \distribution(s_i)) \geq 0 \Big\} \]
where $S=\{s_1,\dots,s_n\}$ are MDP states, $b^i_j$ are real-valued constants and $N_I$ is the number affine inequalities that define $I$. While the values of variables $r_i$ and $b^j_i$ will be computed by our algorithm, $N_I$ is an algorithm parameter which we refer to as the {\em template size}. Note that we require all affine inequalities that specify $I$ to be non-strict. This is the technical requirement for our method to provide relative completeness.

\smallskip\noindent{\bf Input.} The algorithm takes as input an MDP $\MDP = (\States, \Actions, \transitions)$ together with affine sets of initial distributions $\Init$, target distributions $T$ and safe distributions $H$. It also takes as input the template size parameter $N_I$.

\smallskip\noindent{\bf Algorithm outline.} The algorithm employs a template-based synthesis approach and proceeds in three steps. First, it fixes symbolic templates for an initial distribution $\distribution_0 \in \Init$, a memoryless policy $\strategy$, and an affine distributional reach-avoid certificate $\certificate = (R, I)$. Symbolic variables that define the templates are at this stage of {\em unknown} value. Second, the algorithm collects a system of constraints over the symbolic template variables that encode that $\distribution_0 \in \Init$, $\strategy$ is a memoryless policy, and $\certificate$ is a correct distributional reach-avoid certificate. Third, it solves the resulting system of constraints, to get concrete instances of $\distribution_0$, $\strategy$, and $\certificate$. We now detail these steps.

\smallskip\noindent{\bf Step 1 -- Fixing templates.} The algorithm fixes templates for $\distribution_0$, $\strategy$ and $\certificate = (R, I)$:
\begin{compactitem}
	\item {\em Template for $\distribution_0$.} For each MDP state $s_i$, $1\leq i\leq n$, the algorithm introduces a symbolic template variable $m_i$ to encode the probability of initially being in $s_i$.
	\item {\em Template for $\strategy$.} Since the algorithm searches for a memoryless policy, for each state action pair $\state_i \in \States$ and $\action_j \in \Actions$ we fix a symbolic template variable $p_{\state_i,\action_j}$ to encode the probability of taking action $\action_j$ in state $s_i$. If $\action_j \not\in \Actions(s_i)$, we set $p_{\state_i,\action_j}=0$.
	\item {\em Template for $R$.} The template for $R$ is defined by introducing $n+1$ real-valued symbolic template variables $r_0, \dots, r_n$ and letting $R  = r_0 + \sum_{i=1}^n r_i \cdot \mu(s_i)$.
	\item {\em Template for $I$.} The template for $R$ is defined by introducing real-valued symbolic template variables $b^j_i$ for each $1 \leq j \leq N_I$ and $0 \leq i \leq n$, with
	\[ I = \Big\{\distribution \in \Distributions(\States) \mid {\bigwedge}_{j=1}^{N_I} (b_0^j + {\sum}_{i=1}^n b_i^j \cdot \distribution(s_i)) \geq 0 \Big\}. \]
\end{compactitem}

\noindent{\bf Step 2 -- Constraint collection.} The algorithm now collects constraints over the template variables that together encode that $\distribution_0 \in \Init$, $\strategy$ is a memoryless policy and $\certificate$ is a distributional reach avoid certificate. The constraint $\Phi_{\textrm{init}}$ encodes that $\distribution_0 \in \Init$, $\Phi_{\strategy}$ encodes that $\strategy$ is a memoryless policy, whereas $\Phi_1,\dots,\Phi_5$ encode the 5 defining conditions in Definition~\ref{def:certificate}:
\begin{compactitem}
	\item $\Phi_{\textrm{init}} \equiv (\mu_0 \in \Init) \equiv {\bigwedge}_{j=1}^{N_{\Init}} (a_0^j + \sum_{i=1}^n a_i^j \cdot m_i) \bowtie 0$, where recall $a_i^j$'s are real constants defining $\Init$.
	\item $\Phi_{\strategy} \equiv \bigwedge_{i=1}^n \Big( \sum_{j=1}^{|\Actions|} p_{\state_i,\action_j} = 1 \land \bigwedge_{j=1}^{|\Actions|} (p_{\state_i,\action_j} \geq 0) \Big)$
	\item $\Phi_1 \equiv \forall \mathbf{x} \in \mathbb{R}^{n}.\, \Init(\mathbf{x}) \Rightarrow I(\mathbf{x})$.
	\item $\Phi_2 \equiv \forall \mathbf{x} \in \mathbb{R}^{n}.\, I(\mathbf{x}) \Rightarrow I(\mathrm{step}(\mathbf{x}))$, where $\mathrm{step}(\mathbf{x})(x_i) = \sum_{\state_k \in \States, \action_j \in \Actions(s_k)} p_{\state_k, \action_j} \cdot \transitions(\state_k, \action_j, \state_i) \cdot x_j$ yields the distribution after applying one step of policy $\strategy$.
	\item $\Phi_3 \equiv \forall \mathbf{x} \in \mathbb{R}^{n}.\, I(\mathbf{x}) \Rightarrow H(\mathbf{x})$.
	\item $\Phi_4 \equiv \forall \mathbf{x} \in \mathbb{R}^{n}.\, I(\mathbf{x}) \Rightarrow R(\mathbf{x}) \geq 0$.
	\item $\Phi_5 \equiv \forall \mathbf{x} \in \mathbb{R}^{n}.\, I(\mathbf{x}) \backslash T(\mathbf{x}) \Rightarrow R(\mathbf{x}) \geq R(\mathrm{step}(\mathbf{x})) - 1$, where $\mathrm{step}(\mathbf{x})$ is defined as above.
\end{compactitem}

\smallskip\noindent{\bf Step 3 -- Constraint solving.} The initial constraint $\Phi_{\textrm{init}}$ and the policy constraint $\Phi_\pi$ are purely existentially quantified over template variables. However, $\Phi_1,\dots,\Phi_5$ all contain alternation of an existential quantifier over the symbolic template variables followed by a universal quantifier over the distribution $\mathbf{x}\in\mathbb{R}^n$ over the MDP state probabilities. Quantifier alternation over real-valued arithmetic is in general hard to handle directly and can lead to inefficiency in solvers. To that end, our algorithm first translates these constraints into equisatisfiable \emph{purely existentially quantified} constraints, before feeding the resulting constraints to an off-the-shelf solver.

We begin by noting that $\Phi_1,\Phi_2,\Phi_4,\Phi_5$ can all be expressed as conjunctions of finitely many constraints of the form
\begin{eqnarray}\label{eq:horn}
	\begin{split}
	\forall\mathbf{x}\in\mathbb{R}^n.\ &(\textrm{affexp}_1(\mathbf{x}) \bowtie 0) \land \dots \land (\textrm{affexp}_N(\mathbf{x}) \bowtie 0) \\
	&\Longrightarrow (\textrm{affexp}(\mathbf{x}) \geq 0),
	\end{split}
\end{eqnarray}
with each $\textrm{affexp}_i(\mathbf{x})$ and $\textrm{affexp}(\mathbf{x})$ being an affine expression over $\mathbf{x}$ whose affine coefficients are either concrete real values or symbolic template variables, and each $\bowtie \in \{\geq, >\}$. The inequalities on the left-hand-side of the implication may be both strict and non-strict, however the inequalities on the right-hand-side of each of $\Phi_1,\Phi_2,\Phi_4,\Phi_5$ are always non-strict, since we assumed that the template for $I$ is specified via non-strict affine inequalities.
Now, to remove quantifier alternation, we apply the translation of~\cite[Corollary~1]{AsadiC0GM21} which is an extension of Farkas' lemma~\cite{farkas1902theorie} that allows strict inequalities on the left-hand-side of the implication (we provide this translation in Appendix~\ref{app:quantifierelim}). This allows us to translate each constraint of the form as in Eq.~\eqref{eq:horn} into an equisatisfiable purely existentially quantified system of quadratic constraints with real-valued variables, where variables are either symbolic template variables or fresh symbolic variables introduced in translation.

On the other hand, $\Phi_3 \equiv \forall \mathbf{x} \in \mathbb{R}^{n}.\, I(\mathbf{x}) \Rightarrow H(\mathbf{x})$ is a conjunction of constraints of the form
\begin{eqnarray}\label{eq:horn2}
	\begin{split}
		\forall\mathbf{x}\in\mathbb{R}^n.\ &(\textrm{affexp}_1(\mathbf{x}) \geq 0) \land \dots \land (\textrm{affexp}_N(\mathbf{x}) \geq 0) \\
		&\Longrightarrow (\textrm{affexp}(\mathbf{x}) \bowtie 0),
	\end{split}
\end{eqnarray}
since $I$ is specified in terms of non-strict inequalities but $H$ can be specified in terms of both strict and non-strict inequalities. However, since the set of distributions contained in $I$ is topologically closed and bounded as $I \subseteq \Distributions(\States)$, Eq.~\eqref{eq:horn2} is equivalent to the constraint obtained by replacing $\textrm{affexp}(\mathbf{x}) \bowtie 0$ above  by $\textrm{affexp}(\mathbf{x}) \geq \epsilon$, with $\epsilon > 0$ being a newly introduced symbolic variable. The latter constraint is of the form as in Eq.~\eqref{eq:horn}, hence we may again apply the above translation.

Finally, the algorithm feeds the resulting system to an off-the-shelf SMT solver, and any solution gives a concrete instance of initial distribution $\distribution_0\in\Init$, memoryless policy $\strategy$ and affine distributional reach-avoid certificate $\certificate$.

The following theorem (proved in Appendix~\ref{app:thmalgo}) establishes soundness, relative completeness and an upper bound on the computational complexity of our algorithm. Soundness and relative completeness follow from the fact that the quantifier removal procedure yields an equisatisfiable system of constraints. The PSPACE upper bound follows since the quantifier removal procedure reduces the problem to solving a sentence in the existential first-order theory of the reals. 

\begin{theorem}\label{thm:firstalgo}
	\emph{Soundness}: If the algorithm returns initial distribution $\distribution_0\in\Init$, memoryless policy $\strategy$, and affine inductive distributional reach-avoid certificate $\certificate$, then the MDP $\MDP$ satisfies existential $(T,H)$-reach-avoidance under $\strategy$.
	
	\emph{Relative completeness}: If there exists $\distribution_0 \in \Init$, a memoryless policy $\strategy$, and an affine distributional $(T,H)$-reach-avoid certificate $\certificate$ from $\distribution_0$ under $\strategy$, then there exists a template size $N_I \in \mathbb{N}$ such that $\distribution_0$, $\strategy$, and $\certificate$ are computed by the algorithm.
	
	\emph{Complexity}: The runtime of the algorithm is in PSPACE in the size of the encoding of the MDP, $\Init$, $T$, $H$, and the template size parameter $N_I \in \mathbb{N}$. 
\end{theorem}

\noindent{\bf Extension to unit/universal policy synthesis.} We now show how the above algorithm can be extended to solve unit and universal Policy Synthesis problems as well. For unit Policy Synthesis, the initial distribution $\mu_0$ is given, hence in Step~1 of the algorithm we do not need to fix a template for it and in Step~2 we remove the constraint~$\Phi_{\textrm{init}}$. For universal Policy Synthesis, we now need distributional reach-avoidance to hold from all initial distributions in $\Init$. Hence, in Step~1 we again do not need to fix a template for $\mu_0$ whereas in Step~2 we set $\Phi_1 \equiv \forall \mathbf{x} \in \mathbb{R}^{n}.\ (\Init(\mathbf{x}) \Rightarrow I(\mathbf{x}))$. In both cases, the rest of the algorithm proceeds analogously as above. Moreover, as in Theorem~\ref{thm:firstalgo}, both algorithms are sound, relatively complete and of PSPACE runtime in the size of the problem.

\smallskip\noindent{\bf Extension to policy verification.} The above algorithm is also easily extended to solve the unit/existential/universal Policy Verification problems. In particular, in these problems the memoryless policy $\pi$ is given and need not be computed. Thus, in Step~1 above we need not fix the templates $p_{s_i,a_j}$ for the policy as these are given concrete values, and in Step~2 we remove the constraint $\Phi_\pi$. The rest of the algorithms proceed analogously as in the Policy Synthesis case, and as a corollary of Theorem~\ref{thm:firstalgo} we can show that they are sound, relatively complete and of PSPACE runtime in the size of the problem.

\section{Extension to General Policies}\label{sec:algodistributional}

We also extend the algorithms for memoryless policies in Section~\ref{sec:algomemoryless} to sound but incomplete procedures for solving unit/existential/universal Policy Synthesis and Verification problems under distributionally memoryless policies. While memoryless policies are preferred in practice since they can be efficiently deployed and executed, it was shown in~\cite{AkshayCMZ23} that there exist MDPs for which memoryless or even bounded memory policies are insufficient for ensuring distributional safety, and the same example shows that bounded memory strategies are insufficient for distributional reach-avoidance. However, by Theorem~\ref{stm:dist_strat_is_enough}, it always suffices to restrict to distributionally memoryless policies.

The key challenge in extending our algorithms in Section~\ref{sec:algomemoryless} is the design of an appropriate policy template, since now it becomes insufficient to introduce one template variable $p_{\state_i,\action_j}$ for each state-action pair as these probabilities may depend on history. However, it turns out that distributionally memoryless policies do admit a sufficiently simple template specified in terms of quotients of symbolic affine expressions over $\mathbf{x}$.
In Appendix~\ref{app:generalpolicies}, we show how to extend our algorithm for existential Policy Synthesis to compute distributionally memoryless policies. 
We also provide extensions to unit/universal Policy Synthesis and to Policy Verification. Finally, we prove that the algorithm is sound and runs in PSPACE. However, the algorithm does not provide relative completeness guarantees, since our templates for distributionally memoryless policies are not general and only allow affine expressions over distribution probabilities.

\section{Experimental Evaluation}\label{sec:experiments}

We implemented a prototype of our method in Python~3, using \texttt{SymPy} \cite{DBLP:journals/peerj-cs/MeurerSPCKRKIMS17} for symbolic expressions and \texttt{PySMT} \cite{gario2015pysmt} to manage SMT solvers. We employ \texttt{Yices} 2.6 \cite{Dutertre14} as solving back-end.
We also evaluated \texttt{z3} \cite{DBLP:conf/tacas/MouraB08} and \texttt{mathsat} \cite{mathsat5}, but  \texttt{Yices} seemed to consistently perform best.
Our experiments were executed on consumer hardware (AMD Ryzen 3600 CPU with 16 GB RAM). Our implementation is publicly available at \url{https://zenodo.org/records/11082466}.

As mentioned in the Introduction, we are not aware of any existing automated methods for solving this task in the distributional reach-avoidance setting, thus we do not have a reasonable baseline to compare against.
The evaluation of our prototype is aimed at showing that distributional certificates can be found on reasonably sized systems, without heuristics and optimization.

\smallskip\noindent{\bf Benchmarks.}
We evaluate our method on several distributional reach-avoid tasks.
Most are modelling a robot swarm in different gridworld environments.
Each model requires that at some point in time at least $90\%$ of the robots are in the target set of states, while no more than $10\%$ of robots may be in the unsafe set of states at any intermediate step.
Here, we present the following five models:
\textbf{Running}, i.e.\ the environment in Example~\ref{eg:grid-running}, \textbf{Double}, a 3x5 grid where the robots start in two different locations and need to reach two goal states, and three grids of size 5x4, 8x8, and 20x10, comprising various transition dynamics as well as limited and forbidden regions.
For all grid world models, we consider both (unit) Policy Verification task in which a policy is fixed and needs to be verified, and (unit) Policy Synthesis task where a policy together with a certificate needs to be computed.
For all examples, we find the template size $N_I = 1$ to be sufficient.
We also present results on \textbf{Insulin}, a pharmacokinetics system \cite[Ex.~2]{DBLP:journals/jacm/AgrawalAGT15}, based on \cite{DBLP:conf/qest/ChadhaKVAK11}, and \textbf{PageRank} \cite[Fig.~3]{DBLP:journals/jacm/AgrawalAGT15}.
As these two are Markov chains, hence the verification and synthesis tasks coincide.
A more detailed description of all models can be found in Appendix~\ref{app:models}.



\smallskip\noindent{\bf Results.}  Our results are shown in Table~\ref{tab:results}. 
Our prototype is able to solve 7/7 Policy Verification tasks with ease, and 3/5 Policy Synthesis tasks.
A notable feature of this performance is that it is applicable to robot swarms with arbitrarily many agents, where the model size in the state-based view of MDPs would be intractable for classical model checking tools. This shows that our method is capable of solving highly non-trivial distributional reach-avoidance tasks. Furthermore, our results show that memoryless policies are sufficient in many scenarios.

We believe that the reason behind better scalability of our tool on Policy Verification compared to Policy Synthesis is that the final SMT query in Policy Verification tasks is structurally simpler.
In particular, observe that the query for verification of \textbf{Grid} 20x10 is larger than the synthesis query for \textbf{Grid} 8x8, yet it is solved much faster.
To provide further insight, we provide an example SMT query generated for the \textbf{Grid} 5x4 example in Appendix~\ref{app:smt}. Thus, any improvement in SMT solvers to handle larger constraints will improve the scalability of our approach.


\smallskip\noindent{\bf Practical observations.} Somewhat surprisingly, we often observe that the computationally more expensive part of our implementation is construction of constraints (Steps~1 and~2 in Section~\ref{sec:algomemoryless}), especially for the larger certification examples.
We believe that this is due to our naive usage of \texttt{SymPy} to extract constraints, since in theory this procedure should run in polynomial time without any complicated data structures.
In our prototype implementation, we did not aim for efficient extraction and manipulation of affine expressions.
Improvements on this end, e.g.\ by manually implementing a tailored polynomial representation, would further decrease the runtime of our tool.

We also observe that performance of SMT solvers is highly volatile, with their runtimes sometimes increasing 10- or 100-fold on the same instance, presumably due to running into a bad randomized initialization.
Implementing a tighter integration with such solvers and, in particular, providing them with heuristical guidance could further improve performance.



\begin{table}
	\caption{
		Summary of our experiments.
		For each model we list, from left to right, the number of (reachable) states, the number of actions, the time used for invariant generation, the time spent by the SMT solver, and the total number of variables, constraints, and operations in the query sent to the solver.
		The first line for each model is the policy verification query, the second line is the policy synthesis query, where applicable.
		T/O denotes a timeout after 10 minutes.
	}
	\centering\small\setlength{\tabcolsep}{5pt}
	\begin{tabular}{rccccccc}
		              Model & $|\States|$ & Act. & Inv. &  SMT  & Var. & Con. & Ops  \\
		\midrule
		   \textbf{Running} &      7      &  19  &  1s  & $<$1s &  64  &  81  & 556  \\
		                    &             &      &  2s  & $<$1s &  82  & 105  & 776  \\
		    \textbf{Double} &     11      &  30  &  3s  & $<$1s &  88  & 113  & 795  \\
		                    &             &      &  7s  & $<$1s & 115  & 148  & 1114 \\
		  \textbf{Grid} 5x4 &     15      &  29  &  2s  & $<$1s & 112  & 145  & 1018 \\
		                    &             &      &  6s  &  1s   & 133  & 173  & 1395 \\
		  \textbf{Grid} 8x8 &     32      &  99  &  8s  &  1s   & 216  & 284  & 2086 \\
		                    &             &      & 19s  &  T/O  & 312  & 410  & 3144 \\
		\textbf{Grid} 20x10 &     88      & 280  & 68s  &  2s   & 642  & 910  & 6556 \\
		                    &             &      & 238s &  T/O  & 921  & 1276 & 9507 \\
		\midrule
		   \textbf{Insulin} &      5      &  -   &  2s  &  3s   &  74  &  88  & 790  \\
		  \textbf{PageRank} &      5      &  -   &  2s  & $<$1s &  52  &  65  & 571
	\end{tabular}
	\label{tab:results}
\end{table}

\section{Conclusion}
We considered the distributional reach-avoidance problem in MDPs, for which we introduced distributional reach-avoid certificates and proposed fully automated template-based synthesis algorithms for solving policy verification and synthesis problems under distributional reach-avoidance. 
Our work opens several avenues for future work. It would be interesting to consider practical heuristics for template-based synthesis. One could also consider more general distributional properties, ultimately paving the way towards distributional LTL. Finally, our template-based synthesis method assumes that the structure of the template, i.e.~the number of conjunctive clauses in invariants, is provided a priori. This is a known limitation of many template-based synthesis methods, and exploring effective heuristics for template search is an interesting direction.

\section*{Acknowledgements}

This work was supported in part by the ERC-2020-CoG 863818 (FoRM-SMArt), the Singapore Ministry of Education (MOE) Academic Research Fund (AcRF) Tier 1 grant, Google Research Award 2023 and the SBI Foundation Hub for Data and Analytics.

\bibliographystyle{named}
\bibliography{bibliography}

\newpage
\appendix
\begin{center}
	\Large
	Appendix
\end{center}

\section{Proof of Theorem~1}\label{app:distmem}

\begin{theorem*}
	Let $T,H \subseteq \Distributions(\States)$ be target and safe sets. MDP $\MDP$ satisfies unit/existential/universal-$(T,H)$-reach-avoidance under some policy if and only if there exists a distributionally memoryless policy $\strategy$ such that $\MDP$ satisfies satisfies unit/existential/universal-$(T,H)$-reach-avoidance under $\strategy$.
\end{theorem*}

\begin{proof}
	We first prove the theorem for unit and existential $(T,H)$-reach-avoidance. We then modify the proof to extend it to the universal-$(T,H)$-reach-avoidance case as well.
	\begin{enumerate}
		\item {\em Unit/existential-$(T,H)$-reach-avoidance} Suppose that $\MDP$ satisfies unit/existential-$(T,H)$-reach-avoidance from $\mu_0$ under some policy. In order to prove the theorem, we need to show that there exists a distributionally memoryless policy $\strategy$ such that $\MDP$ satisfies unit/existential-$(T,H)$-reach-avoidance from $\mu_0$ under $\strategy$. We obtain a distributionally memoryless policy $\pi$ as follows.
		
		Let $\pi'$ be a policy such that $\MDP$ satisfies unit/existential-$(T,H)$-reach-avoidance from $\mu_0$ under $\pi'$ and such that the target set is reached under $\pi'$ in the smallest possible number of steps among all other policies under which unit/existential-$(T,H)$-reach-avoidance from $\mu_0$ is satisfied.
		Let $\distribution_0, \distribution_1, \distribution_2, \dots$ be the stream of distributions induced by $\pi'$ from $\distribution_0$. Since $\MDP$ satisfies $(T,H)$-reach-avoidance from $\distribution_0$ under $\strategy'$, there exists index $N \in \mathbb{N}$ such that $\distribution_N \in T$ and $\distribution_0,\dots,\distribution_{N-1} \in H \backslash T$.
		
		First, we claim that $\distribution_0,\dots,\distribution_{N}$ are all distinct distributions. Suppose that, on the contrary, there exist $0 \leq i < j \leq N$ such that $\mu_i = \mu_j$. To obtain contradiction, we use $\pi'$ to construct another policy $\pi''$ such that $\MDP$ satisfies unit/existential-$(T,H)$-reach-avoidance from $\mu_0$ under $\pi''$ and such that the target set is reached under $\pi''$ in a strictly smaller number of steps compared to $\pi'$. Define $\pi'': \FPaths<\MDP> \to \Distributions(\Actions)$ as follows. For a finite history $\rho$ of length $t < i$ or $t \geq N - j + i$, we let $\pi''(\rho) = \pi'(\rho)$. On the other hand, for a finite history $\rho$ of length $i \leq t <  N - j + i$, we do the following. For each $t \in \mathbb{N}$ and for each state $\state$ and action $\action$, let $p^t_{s, a} = \ProbabilityMC<\MDP^{\strategy'} \distribution_0>[\mathsf{act}_t(\infinitepath) = a]$ be the probability of playing action $a$ at step $t$ under $\pi'$, where $\mathsf{act}_t : \IPaths<\MDP> \to \Actions$ yields the $t$-th action of an infinite path. Then, we define
		\[ \pi''(\finitepath)(a) = \frac{p^{t+j-i}_{s, a}}{\sum_{\action' \in \Actions(s_{t+j-i})}p^{t+j-i}_{s, a'}}. \]
		In words, $\pi''$ picks each action with the probabilities that transform $\mu_{t+j-i}$ into $\mu_{t+j-i+1}$. This ensures that the stream of distributions induced under by $\pi''$ is $\distribution_0,\dots,\distribution_i,\distribution_{j+1},\dots,\distribution_N$. Therefore, $\MDP$ satisfies	the unit/existential-$(T,H)$-reach-avoidance from $\mu_0$ under $\pi''$ and the induced stream reaches $T$ in a strictly smaller number of steps compared to $\pi'$. This yields a contradiction, hence we have proved that  $\distribution_0,\dots,\distribution_{N}$ are all distinct.
		
		Finally, we use $\pi'$ to construct a distributionally memoryless policy $\pi: \FPaths<\MDP> \to \Distributions(\Actions)$ under which $\MDP$ satisfies unit/existential-$(T,H)$-reach-avoidance from $\mu_0$. On histories of length $t \geq N$, $\pi$ can be an arbitrary distributionally memoryless policy. On the other hand, we proceed by induction on $0 \leq t < N$ to define $\pi$ on finite paths of length $t$. For each finite path $\finitepath = s_0 a_0 s_1 a_1 \dots s_t$, if $\distribution_t(s_t) = 0$ we define $\pi(\finitepath)$ arbitrarily. Else, for each action $\action \in \Actions(s_t)$, we define
		\[ \pi(\finitepath)(a) = \frac{p^t_{s, a}}{\sum_{\action' \in \Actions(s_t)}p^t_{s, a'}}. \]
		In words, $\pi$ is a distributionally memoryless policy that transforms $\distribution_t$ into $\distribution_{t+1}$ by prescribing the same probabilities to actions as $\pi'$ does. Hence, $\MDP^{\strategy}(\distribution_i) = \distribution_{i+1}$ and so $\pi$ induces the same stream of distributions $\distribution_0, \distribution_1, \dots, \distribution_N$ which by assumption satisfies $(T,H)$-reach-avoidance. Since $\distribution_0,\dots,\distribution_{N}$ are all distinct, this construction is indeed well-defined. Thus, our construction ensures that $\pi$ is distributionally memoryless, which concludes the proof.
		
		\item {\em Universal-$(T,H)$-reach-avoidance} Suppose that there exists a policy $\pi'$ in $\MDP$ such that, for every initial distribution $\mu_0 \in \Init$, $\MDP$ satisfies $(T,H)$-reach-avoidance from $\mu_0$ under $\pi'$. In order to prove the theorem, we need to show that there exists a distributionally memoryless policy $\strategy$ such that, for every $\mu_0 \in \Init$, $\MDP$ satisfies $(T,H)$-reach-avoidance from $\mu_0$ under $\strategy$. We obtain a distributionally memoryless policy $\pi$ from $\pi'$ as follows.
		
		For each $\mu_0$, let $\textsc{stream}^T_{\mu_0,\pi'}$ be the finite prefix of the stream induced by $\mu_0$ under $\pi'$ until $T$ is reached. We may without loss of generality assume that all distributions contained in $\textsc{stream}^T_{\mu_0,\pi'}$ are distinct, for each $\mu_0 \in \Init$. To see this, observe that we may inductively refine $\pi'$ on histories of length $i=0,1,\dots$ to ensure that the first $i$ distributions in each stream are pairwise distinct. The refinement procedure is analogous to the construction of $\pi''$ from $\pi'$ in the proof above, in the sense that whenever a repeated distribution occurs in a stream, the refined policy removes the induced cycle from the stream.
		
		Before proving the claim, we introduce some additional notation. First, let
		\[ \textsc{Reach}^T_{\pi'} = \cup_{\mu_0 \in \Init} \, \textsc{stream}^T_{\mu_0,\pi'} \]
		be the set of all reachable distributions from any initial distribution $\mu_0 \in \Init$ under $\pi'$ until the target set $T$ is reached. Next, for each $\mu \in \textsc{Reach}^T_{\pi'}$, define its {\em distance} to $T$ as the smallest number of steps in which a distribution in $T$ is reached from $\mu$ along any stream $\textsc{stream}^T_{\mu_0,\pi'}$ that contains $\mu$. Finally, for each $\mu \in \textsc{Reach}^T_{\pi'}$, let $\textsc{min-stream}(\mu)$ be a stream that contains $\mu$ and achieves the minimum distance from $\mu$ to $T$ (in case of several such streams, we pick one for each $\mu$). Then, if $\mu \in \textsc{Reach}^T_{\pi'} \backslash T$, define $\textsc{Next}(\mu)$ to be the sussessive distribution of $\mu$ in the stream $\textsc{min-stream}(\mu)$. On the other hand, if $\mu \in \textsc{Reach}^T_{\pi'} \cap T$, then $\mu$ is the last element of $\textsc{min-stream}(\mu)$ and so we do not define $\textsc{Next}(\mu)$.
		
		We are now ready to define the distributionally memoryless policy $\pi$. Consider an initial distribution $\mu_0 \in \Init$ and a finite path $\finitepath = s_0 a_0 s_1 a_1 \dots s_t$. Let $\mu_t$ be the $t$-th distribution in the stream induced from $\mu_0$ by $\pi'$. If $\mu_t \not\in \textsc{Reach}^T_{\pi'}$, define $\pi(\finitepath)$ arbitrarily such that $\pi$ is distributionally memoryless. Otherwise, let $(P_{s,a})_{s\in\States,a\in\Actions}$ be a matrix that transforms $\mu_t$ into $\textsc{Next}(\mu_t)$ under $\pi'$, i.e.~such that $\textsc{Next}(\mu_t) = P \cdot \mu_t$. Note that $P$ exists and is well-defined: since $\mu_t$ and $\textsc{Next}(\mu_t)$ are successive distributions in the stream $\textsc{stream}(\mu_t)$ under $\pi'$, $P$ is the distribution transformer induced by $\pi'$ in this stream. Then, if $\mu_t(s_t) = 0$, we define $\pi(\finitepath)$ arbitrarily. Else, for each action $\action \in \Actions(s_t)$, we define
		\[ \pi(\finitepath)(a) = \frac{P_{s, a}}{\sum_{\action' \in \Actions(s_t)}P_{s, a'}}. \]
		In words, $\pi$ is a distributionally memoryless policy that transforms $\mu_t$ into $\textsc{Next}(\mu_t)$ for every finite history $\finitepath = s_0 a_0 s_1 a_1 \dots s_t$ by prescribing the same probabilities to actions as $\pi'$ does.
		
		We are left to prove that the construction of our distributional memoryless policy $\pi$ also ensures that, for every $\mu_0 \in \Init$, $\MDP$ satisfies $(T,H)$-reach-avoidance from $\mu_0$ under $\strategy$. To see this, fix $\mu_0 \in \Init$. Then, the stream of distributions induced from $\mu_0$ under $\pi$ is
		\[ \mu_0, \textsc{Next}(\mu_0), \textsc{Next}(\textsc{Next}(\mu_0)), \dots \]
		Note that, since each distribution in the above stream belongs to stream $\textsc{stream}^T_{\mu_0,\pi'}$ for some $\mu_0$ and since $\textsc{stream}^T_{\mu_0,\pi'} \subseteq H$ for each $\mu_0 \in \Init$ due to $\MDP$ satisfying $(T,H)$-reach-avoidance from $\mu_0$ under $\pi'$, we conclude that this stream is also fully contained in $H$. On the other hand, our definition of $\textsc{Next}(\mu_0)$ which chooses a successor that minimizes distance to $T$ also ensures that the distance to $T$ strictly decreases along the stream $ \mu_0, \textsc{Next}(\mu_0), \textsc{Next}(\textsc{Next}(\mu_0)), \dots$. Hence, this stream must eventually reach a distribution in $T$, showing that $\MDP$ satisfies $(T,H)$-reach-avoidance from $\mu_0$ under $\strategy$. Since the initial distribution $\mu_0$ was arbitrary, this concludes our proof.
	\end{enumerate}
\end{proof}

\section{Proof of Theorem~2}\label{app:certificate}

\begin{theorem*}[Sound and complete certificates]
	Let $\MDP$ be an MDP, $\distribution_0 \in \Init$ and $\strategy$ be a distributionally memoryless policy. Then $\MDP$ satisfies
	\begin{compactenum}
		\item unit-$(T,H)$-reach-avoidance under $\strategy$ iff there exists a distributional $(T,H)$-reach-avoid certificate for $\MDP$ from $\distribution_0$ under $\strategy$.
		\item existential-$(T,H)$-reach-avoidance under $\strategy$ iff there exists a $\mu_0\in\Init$ and distributional $(T,H)$-reach-avoid certificate for $\MDP$ from $\distribution_0$ under $\strategy$.
		\item universal-$(T,H)$-reach-avoidance under $\strategy$ iff there exists a universal distributional $(T,H)$-reach-avoid certificate for $\MDP$ under $\strategy$.
	\end{compactenum}
\end{theorem*}

\begin{proof}
	We first prove the theorem for unit and existential $(T,H)$-reach-avoidance. We then modify the proof to extend it to the universal-$(T,H)$-reach-avoidance case as well.
	\begin{enumerate}
		\item {\em Unit/existential-$(T,H)$-reach-avoidance.}
		\begin{enumerate}
			\item {\em Proof of $\Longrightarrow$ direction.} Suppose that $\MDP$ satisfies unit/existential-$(T,H)$-reach-avoidance from $\mu_0 \in \Init$ under some distributionally memoryless policy $\strategy$. We need to show that there exists a distributional $(T,H)$-reach-avoid certificate for $\MDP$ from $\mu_0$ under~$\strategy$.
			
			Let $\distribution_0,\distribution_1,\distribution_2,\dots$ be the stream of distributions induced by policy $\strategy$ from the initial distribution $\distribution_0$. Then there exists the smallest $N\in\mathbb{N}$ such that $\distribution_N \in T$ and $\distribution_0,\dots,\distribution_{N-1} \in H \backslash T$. Define the distributional invariant $I \subseteq \Distributions(S)$ via
			\[ I = \{\distribution_0, \distribution_1, \dots, \distribution_N\} \]
			and the distributional ranking function $R: \Distributions(S) \rightarrow \mathbb{R}$ via
			\[ R(\mu) = \begin{cases}
				\max\{0, N - i \}, &\text{if } \distribution = \distribution_i \text{ for some } i \\
				0, &\text{otherwise}
			\end{cases}
			\]
			Note that, since $\pi$ is distributionally memoryless, all $\distribution_0,\distribution_1,\dots,\distribution_N$ must be distinct as otherwise we would have a cycle and $\distribution_N \in T$ would not be reached. Hence, $R$ is well-defined.
			One easily verifies by inspection that $\certificate = (R, I)$ satisfies all the conditions in the definition of distributional reach-avoid certificates. Indeed, $I = \{\distribution_0, \distribution_1, \dots, \distribution_N\}$ contains the initial distribution $\distribution_0$, is inductive until $T$ is reached and is contained within the safe set $H$. On the other hand, $R$ is nonnegative within $I$ and strictly decreases by $1$ until $\distribution_N \in T$ is reached. Hence, $\certificate = (R, I)$ satisfies all the defining conditions, as claimed.
			
			\item {\em Proof of $\Longleftarrow$ direction.} Suppose now that there exists a distributional $(T,H)$-reach-avoid certificate for $\MDP$ from $\mu_0$ under some distributionally memoryless policy $\strategy$. We need to show that $\MDP$ satisfies unit/existential-$(T,H)$-reach-avoidance from $\mu_0$ under $\strategy$.
			
			Let $\distribution_0,\distribution_1,\distribution_2,\dots$ be the stream of distributions induced by policy $\strategy$ from the initial distribution $\distribution_0$, and consider the sequence $R(\distribution_0), R(\distribution_1), R(\distribution_2), \dots$ of values of the ranking function $R$ along this sequence. By conditions~4 and~5 in the Definition of distributional reach-avoid certificates, the sequence is nonnegative and it decreases by at least $1$ at every step until a distribution in $T$ is reached. Hence, as a sequence of nonnegative numbers cannot decrease by at least $1$ indefinitely, there exists the smallest $N\in\mathbb{N}$ such that $\distribution_N \in T$ and $\distribution_0,\dots,\distribution_{N-1} \not\in T$. 
			
			Now, by condition~1 in the definition it follows that $\distribution_0 \in I$. Furthermore, by condition~2 and by induction on $i$, it follows that $\distribution_i \in I \backslash T$ for each $0 \leq i \leq N-1$ since we showed above that $\distribution_0,\dots,\distribution_{N-1} \not\in T$. Finally, by condition~3, since $I \subseteq H$ it follows that $\distribution_0,\dots,\distribution_{N-1} \in I \subseteq H$. Hence, the stream $\distribution_0,\distribution_1,\distribution_2,\dots$ reaches a distribution in $T$ while staying within $H$ in the process. This concludes the proof that  $\MDP$ satisfies unit/existential-$(T,H)$-reach-avoidance from $\mu_0$ under $\strategy$.
		\end{enumerate}
		
		\item {\em Universal-$(T,H)$-reach-avoidance.}
		\begin{enumerate}
			\item {\em Proof of $\Longrightarrow$ direction.} Suppose that $\MDP$ satisfies universal-$(T,H)$-reach-avoidance under some distributionally memoryless policy $\strategy$. We need to show that there exists a universal distributional $(T,H)$-reach-avoid certificate for $\MDP$ under $\strategy$.
			
			For each initial distribution $\mu_0 \in \Init$, let $\distribution_0,\distribution_1,\distribution_2,\dots$ be the stream of distributions induced from $\mu_0$ under $\strategy$. Then there exists the smallest $N\in\mathbb{N}$ such that $\distribution_N \in T$ and $\distribution_0,\dots,\distribution_{N-1} \in H \backslash T$. Let $\textsc{stream}^T_{\mu_0,\pi} = \{\mu_0,\dots,\mu_T\}$.
			
			Define the distributional invariant $I \subseteq \Distributions(S)$ via
			\[ I = \cup_{\mu_0\in\Init}\, \textsc{stream}^T_{\mu_0,\pi}. \]
			Then, for each $\mu \in I$, define the {\em distance} $\textsc{distance}^T_{\pi}(\mu)$ of $\mu$ from $T$ to be the number of steps when applying the distributionally memoryless policy $\pi$ to $\mu$ until a distribution in $T$ is reached. Finally, define the distributional ranking function $R: \Distributions(S) \rightarrow \mathbb{R}$ via
			\[ R(\mu) = \begin{cases}
				\textsc{distance}^T_{\pi}(\mu), &\text{if } \mu \in I \\
				0, &\text{otherwise}
			\end{cases}
			\]
			One easily verifies by inspection that $\certificate = (R, I)$ satisfies all the conditions in the definition of universal distributional reach-avoid certificates. Indeed, $I$ contains all initial distribution in $\Init$ and is inductive under $\pi$ until $T$ is reached since it contains all streams induced by $\pi$ when starting in any initial distribution in $\Init$. Moreover, we have $I \subseteq H$ since each $\textsc{stream}^T_{\mu_0,\pi} \subseteq H$, due to $\MDP$ satisfying universal-$(T,H)$-reach-avoidance under $\pi$. Finally, $R$ is nonnegative within $I$ and strictly decreases by $1$ until a distribution in $T$ is reached. Hence, $\certificate = (R, I)$ satisfies all the defining conditions, as claimed.
			
			\item {\em Proof of $\Longleftarrow$ direction.} Suppose now that there exists a universal distributional $(T,H)$-reach-avoid certificate for $\MDP$ under some distributionally memoryless policy $\strategy$. We need to show that $\MDP$ satisfies universal-$(T,H)$-reach-avoidance under $\strategy$.
			
			Let $\mu_0 \in \Init$ be an initial distribution and let $\distribution_0,\distribution_1,\distribution_2,\dots$ be the stream of distributions induced by policy $\strategy$ from the initial distribution $\distribution_0$. Consider the sequence $R(\distribution_0), R(\distribution_1), R(\distribution_2), \dots$ of values of the ranking function $R$ along this sequence. By conditions~4 and~5 in the Definition of distributional reach-avoid certificates, the sequence is nonnegative and it decreases by at least $1$ at every step until a distribution in $T$ is reached. Hence, as a sequence of nonnegative numbers cannot decrease by at least $1$ indefinitely, there exists the smallest $N\in\mathbb{N}$ such that $\distribution_N \in T$ and $\distribution_0,\dots,\distribution_{N-1} \not\in T$. 
			
			Now, by condition~1 in the definition it follows that $\distribution_0 \in I$. Furthermore, by condition~2 and by induction on $i$, it follows that $\distribution_i \in I \backslash T$ for each $0 \leq i \leq N-1$ since $\distribution_0,\dots,\distribution_{N-1} \not\in T$. Finally, by condition~3, since $I \subseteq H$ it follows that $\distribution_0,\dots,\distribution_{N-1} \in I \subseteq H$. Hence, the stream $\distribution_0,\distribution_1,\distribution_2,\dots$ reaches a distribution in $T$ while staying within $H$ in the process. This concludes the proof that  $\MDP$ satisfies unit-$(T,H)$-reach-avoidance from $\mu_0$ under $\strategy$ for every $\mu_0 \in \Init$. Since the initial distribution $\mu_0$ was arbitrary, we conclude that $\MDP$ satisfies universal-$(T,H)$-reach-avoidance under $\strategy$.
		\end{enumerate}
	\end{enumerate}
\end{proof}

\section{Quantifier Elimination in Section~5}\label{app:quantifierelim}

In order to remove quantifier alternation in constraints constructed in Step~2 of the algorithm in Section~5, we first observe that constraints $\Phi_1,\Phi_2,\Phi_4,\Phi_5$ constructed in Step~2 can all be expressed as conjunctions of finitely many constraints of the form
\begin{eqnarray}\label{eq:hornapp}
	\begin{split}
		\forall\mathbf{x}\in\mathbb{R}^n.\ &(\textrm{affexp}_1(\mathbf{x}) \bowtie 0) \land \dots \land (\textrm{affexp}_N(\mathbf{x}) \bowtie 0) \\
		&\Longrightarrow (\textrm{affexp}(\mathbf{x}) \geq 0),
	\end{split}
\end{eqnarray}
with each $\textrm{affexp}_i(\mathbf{x})$ and $\textrm{affexp}(\mathbf{x})$ being an affine expression over $\mathbf{x}$ whose affine coefficients are either concrete real values or symbolic template variables, and each $\bowtie \in \{\geq, >\}$. Note that inequalities on the left-hand-side of the implication may be both strict and non-strict, however the inequalities on the right-hand-side of each of $\Phi_1,\Phi_2,\Phi_4,\Phi_5$ are always non-strict. This is because we assumed that the template for $I$ is specified via non-strict affine inequalities.

The algorithm now translates each constraint of the form as in eq.~\eqref{eq:hornapp} into a system of equisatisfiable and purely existentially quantified constraints. The translation is based on the result of~\cite[Corollary~1]{AsadiC0GM21} which is an extension of Farkas' lemma~\cite{farkas1902theorie} that allows strict inequalities on the left-hand-side of the implication. 

\begin{lemma}[\cite{AsadiC0GM21}]\label{lemma:strengthenedfarkas}
	Let $\mathcal{X} = \{x_1,\dots,x_n\}$ be real-valued variables, and consider a system of affine inequalities over $\mathcal{X}$:
	\begin{equation*}
		\Phi : \begin{cases}
			c^1_0 + c^1_1\cdot x_1 + \dots + c^1_n \cdot x_n \bowtie 0 \\
			\qquad \qquad \qquad \vdots \\
			c^N_0 + c^N_1\cdot x_1 + \dots + c^N_n \cdot x_n \bowtie 0 \\
		\end{cases}.
	\end{equation*}
	with each $\bowtie \in \{\geq, >\}$. Suppose that $\Phi$ is satisfiable.
	Then $\Phi$ entails a non-strict affine inequality $\phi \equiv c_0 + c_1 \cdot x_1 + \dots + c_n \cdot x_n \geq 0$, i.e.\ $\Phi \Longrightarrow \phi$, if and only if $\phi$ can be written as a \emph{non-negative} linear combination of $1 \geq 0$ and affine inequalities in $\Phi$, i.e.\ if and only if there exist $y_0,y_1,\dots,y_n\geq 0$ such that $c_0 = y_0 + \sum_{j=1}^N y_j \cdot c^j_0$, $c_1 = \sum_{j=1}^N y_j \cdot c^j_1$, \dots, $c_n = \sum_{j=1}^N y_j \cdot c^j_n$.
\end{lemma}

The algorithm applies Lemma~\ref{lemma:strengthenedfarkas} to a constraint of the form as in eq.~\eqref{eq:hornapp} as follows. Suppose that 
\[ \textrm{affexp}_i(\mathbf{x}) = w_0^i + w_1^i \cdot x_1 + \dots + w_n^i \cdot x_n \]
for each $1 \leq i \leq N$, and
\[ \textrm{affexp}(\mathbf{x}) = w_0+ w_1 \cdot x_1 + \dots + w_n \cdot x_n. \]
The algorithm introduces fresh symbolic variables $y_0,y_1,\dots,y_N$. Then, the constraint in eq.~\eqref{eq:hornapp} is translated into an equisatisfiable system of constraints
\begin{equation*}
	\begin{split}
		&(y_0 \geq 0) \land (y_1 \geq 0) \land \dots \land (y_N \geq 0) \land \\
		&\land w_0 = y_0 + \sum_{j=1}^N y_j \cdot w^j_0 \\
		&\bigwedge_{i=1}^n  w_i = \sum_{j=1}^N y_j \cdot w^j_i.
	\end{split}
\end{equation*}
Intuitively, this system constrains $y_i's$ to be nonnegative and $\textrm{affexp}$ to be an affine combination of $\textrm{affexp}_1,\dots, \textrm{affexp}_N$ with affine coefficients given by $y_0,y_1,\dots,y_N$. By Lemma~\ref{lemma:strengthenedfarkas}, this yields a system of constraints with real-valued variables which is equisatisfiable to the constraint in eq.~\eqref{eq:hornapp}. Furthermore, the resulting system of constraints is purely existentially quantified. Finally, all resulting constraints are at most quadratic, due to products of symbolic template variables and fresh variables $y_i$ introduced by the translation. This concludes our quantifier elimination procedure.

The constraint $\Phi_3 \equiv \forall \mathbf{x} \in \mathbb{R}^{n}.\, I(\mathbf{x}) \Rightarrow H(\mathbf{x})$ is a conjunction of constraints of the form
\begin{eqnarray}\label{eq:horn2}
	\begin{split}
		\forall\mathbf{x}\in\mathbb{R}^n.\ &(\textrm{affexp}_1(\mathbf{x}) \geq 0) \land \dots \land (\textrm{affexp}_N(\mathbf{x}) \geq 0) \\
		&\Longrightarrow (\textrm{affexp}(\mathbf{x}) \bowtie 0),
	\end{split}
\end{eqnarray}
since $I$ is specified in terms of non-strict inequalities but $H$ can be specified in terms of both strict and non-strict inequalities. However, since the set of distributions contained in $I$ is topologically closed and bounded as $I \subseteq \Distributions(\States)$, eq.~\eqref{eq:horn2} is equivalent to the constraint obtained by replacing $\textrm{affexp}(\mathbf{x}) \bowtie 0$ above  by $\textrm{affexp}(\mathbf{x}) \geq \epsilon$, with $\epsilon > 0$ being a newly introduced symbolic variable. The latter constraint is of the form as in eq.~\eqref{eq:hornapp}, hence we may again apply the above translation.

\section{Proof of Theorem~3}\label{app:thmalgo}

\begin{theorem*}
	\emph{Soundness}: If the algorithm returns initial distribution $\distribution_0\in\Init$, memoryless policy $\strategy$ and affine inductive distributional reach-avoid certificate $\certificate$, then, the MDP $\MDP$ satisfies existential-$(T,H)$-reach-avoidance under $\strategy$.
	
	\emph{Relative completeness}: If there exists $\distribution_0 \in \Init$, a memoryless policy $\strategy$ and an affine distributional $(T,H)$-reach-avoid certificate $\certificate$ from $\distribution_0$ under $\strategy$, then there exists a template size $N_I \in \mathbb{N}$ such that $\distribution_0$, $\strategy$ and $\certificate$ are computed by the algorithm.
	
	\emph{Complexity}: The runtime of the algorithm is in PSPACE in the size of the encoding of the MDP, $\Init$, $T$, $H$ and the template size parameter $N_I \in \mathbb{N}$. 
\end{theorem*}

\begin{proof}
	{\em Soundness.}  Suppose that the algorithm returns an initial distribution $\distribution_0\in\Init$, a memoryless policy $\strategy$ and an affine inductive distributional reach-avoid certificate $\certificate$. This means that Step~3 of the algorithm computed a solution $(\distribution_0, \strategy, R, I, y_0, y_1, \dots, y_N)$ to the system of constraints constructed in Step~3 upon quantifier elimination. By Lemma~\ref{lemma:strengthenedfarkas} and the equisatisfiability of two systems of constraints, this means that $(\distribution_0, \strategy, R, I)$ form a solution to the system of constraints constructed in Step~2. Since this system encodes that $\distribution_0 \in \Init$, that $\strategy$ is a memoryless strategy and that $(R,I)$ form a correct distributional reach-avoid certificate, we conclude that theorem follows.
	
	\smallskip\noindent{\em Completeness.} Suppose that there exist an initial distribution $\distribution_0 \in \Init$, a memoryless policy $\strategy$ and an affine distributional $(T,H)$-reach-avoid certificate $\certificate = (R, I)$ for $\distribution_0$ under $\strategy$. We need to show that there exist $y_0, y_1,\dots, y_N \geq 0$ and the minimal template size $N_I$ such that $(\distribution_0, \strategy, R, I, y_1, \dots, y_N)$ is a solution to the system of constraints in Step~3 for the template size $N_I$. To prove this, let $N_I$ be the number of affine inequalities in $I$. Since $\certificate = (R, I)$ is a correct instance of an affine distributional reach-avoid certificate for $\distribution_0$ under $\strategy$, it satisfies all the constraints in the system constructed by Step~2 of the algorithm. Hence, by Lemma~\ref{lemma:strengthenedfarkas} and the equisatisfiability of the system obtained upon quantifier elimination, it follows that there exist $y_0,y_1,\dots, y_N \geq 0$ such that $(\distribution_0,\strategy, R, I, y_0, y_1, \dots, y_N)$ is a solution to the system of constraints in Step~3 for the template size $N_I$. The theorem follows.
	
	\smallskip\noindent{\em Complexity.} Steps~1 and~2 as well as the quantifier elimination procedure run in time which is polynomial in the size of the input and yield a system of constraints which is polynomial in the size of the input. Thus, the query that is given to an SMT-solver in a sentence in the existential first-order theory of the reals of size polynomial in the problem input size, which can be solved in PSPACE. Therefore, the combined computational complexity of the algorithm is in PSPACE.
\end{proof}

\section{Algorithm for General Policies}\label{app:generalpolicies}

We now present the details behind our algorithm for distributionally memoryless policies in Section~6.

\paragraph{Input.} The algorithm takes as input an MDP $\MDP = (\States, \Actions, \transitions)$ together with affine sets of initial distributions $\Init$, target distributions $T$ and safe distributions $H$. It also takes as input the template size parameter $N_I$.

\paragraph{Algorithm Outline.} Analogously as in Section~5, the algorithm employs a template-based synthesis approach and proceeds in three steps. First, it fixes symbolic templates for an initial distribution $\distribution_0 \in \Init$, a {\em distributionally memoryless} policy $\strategy$ and an affine distributional reach-avoid certificate $\certificate = (R, I)$. Second, the algorithm collects a system of constraints over the symbolic template variables that encode that $\distribution_0 \in \Init$, that $\strategy$ is a distributionally memoryless policy and that $\certificate$ is a correct distributional reach-avoid certificate. Third, it solves the resulting system of constraints, with any solution giving rise to a concrete instance of an initial distribution $\distribution_0 \in \Init$, a distributionally memoryless policy $\strategy$ and an affine distributional reach-avoid certificate $\certificate$.

\paragraph{Step 1 -- Fixing Templates.} The algorithm fixes templates for $\distribution_0$, $\strategy$ and $\certificate = (R, I)$:
\begin{compactitem}
	\item {\em Template for $\distribution_0$.} For each MDP state $s_i$, $1\leq i\leq n$, the algorithm introduces a symbolic template variable $m_i$ to encode the probability of initially being in $s_i$.
	\item {\em Template for $\strategy$.} For each state $\state_i$ and action $\action_j$ pair, we treat $p_{\state_i,\action_j}$ as a function $p_{\state_i,\action_j}(\mathbf{x})$ of the current probability distribution over MDP states and set its template to be a quotient of two affine expressions over $\mathbf{x}$, i.e.
	\[ p_{\state_i,\action_j}(\mathbf{x}) = \frac{\alpha^j_0 + \alpha^j_1 \cdot x_1 + \dots + \alpha^j_n \cdot x_n}{\alpha_0 + \alpha_1 \cdot x_1 + \dots + \alpha_n \cdot x_n}, \]
	where $\alpha^j_i$'s and $\alpha_i$'s are the symbolic template variables. The coefficients in the numerator depend both on the state $\state_i$ and the action $\action_j$, whereas the coefficients in the denominators are simply used to normalize the probabilities to be in $[0,1]$ and to sum to $1$ for each $\state_i$. In order to simplify notation later, for each state $\state_i$ and action $\action_j$ we write
	\[ \textrm{num}_{\state_i,\action_j}(\mathbf{x}) = \alpha^j_0 + \alpha^j_1 \cdot x_1 + \dots + \alpha^j_n \cdot x_n \]
	and
	\[ \textrm{den}_{\state_i}(\mathbf{x}) = \alpha_0 + \alpha_1 \cdot x_1 + \dots + \alpha_n \cdot x_n \]	
	\item {\em Template for $R$.} The template for $R$ is defined by introducing $n+1$ real-valued symbolic template variables $r_0, \dots, r_n$ and letting $R  = r_0 + \sum_{i=1}^n r_i \cdot \mu(s_i)$.
	\item {\em Template for $I$.} The template for $R$ is defined by introducing real-valued symbolic template variables $b^j_i$ for each $1 \leq j \leq N_I$ and $0 \leq i \leq n$, with
	\[ I = \Big\{\distribution \in \Distributions(\States) \mid {\bigwedge}_{j=1}^{N_I} (b_0^j + \sum_{i=1}^n b_i^j \cdot \distribution(s_i)) \geq 0 \Big\}. \]
\end{compactitem}

\paragraph{Step 2 -- Constraint Collection.} The algorithm now collects constraints over the template variables that together encode that $\distribution_0 \in \Init$, that $\strategy$ is a memoryless policy and that $\certificate$ is a distributional reach avoid certificate. The constraint $\Phi_{\textrm{init}}$ encodes that $\distribution_0 \in \Init$, $\Phi_{\strategy}$ encodes that $\strategy$ is a memoryless policy, whereas $\Phi_1,\dots,\Phi_5$ encoding the five defining conditions of distributional reach-avoid certificates:
\begin{compactitem}
	\item $\Phi_{\textrm{init}} \equiv (\mu_0 \in \Init) \equiv {\bigwedge}_{j=1}^{N_{\Init}} (a_0^j + \sum_{i=1}^n a_i^j \cdot m_i) \bowtie 0$, where recall $a_i^j$'s are real constants defining $\Init$.
	\item The definition of $\Phi_{\strategy}$ requires particular care, since we need to ensure that the template defined by a quotient of two symbolic affine expressions indeed defines valid probability values and that it cannot induce division by $0$. Thus, we define the constraint
	\begin{equation*}
		\begin{split}
			\Phi_{\strategy} \equiv &\forall \mathbf{x} \in \mathbb{R}^{n}.\, I(\mathbf{x}) \Rightarrow  \\
			&\bigwedge_{i=1}^n \Big(\textrm{den}_{\state_i}(\mathbf{x}) \geq 1 \\
			&\land \sum_{j=1}^{|\Actions|} \textrm{num}_{\state_i,\action_j}(\mathbf{x}) = \textrm{den}_{\state_i}(\mathbf{x}) \\
			&\land \bigwedge_{j=1}^{|\Actions|} (\textrm{num}_{\state_i,\action_j}(\mathbf{x})  \geq 0) \Big)
		\end{split}
	\end{equation*}
	\item $\Phi_1 \equiv \forall \mathbf{x} \in \mathbb{R}^{n}.\, \Init(\mathbf{x}) \Rightarrow I(\mathbf{x})$.
	\item $\Phi_2 \equiv \forall \mathbf{x} \in \mathbb{R}^{n}.\, I(\mathbf{x}) \Rightarrow I(\mathrm{step}(\mathbf{x}))$, where $\mathrm{step}(\mathbf{x})(x_i) = \sum_{\state_k \in \States, \action_j \in \Actions(s_k)} p_{\state_k, \action_j} \cdot \transitions(\state_k, \action_j, \state_i) \cdot x_j$ yields the distribution after applying one step of the policy $\strategy$ to the distribution $\mathbf{x}$.
	\item $\Phi_3 \equiv \forall \mathbf{x} \in \mathbb{R}^{n}.\, I(\mathbf{x}) \Rightarrow R(\mathbf{x}) \geq 0$.
	\item $\Phi_4 \equiv \forall \mathbf{x} \in \mathbb{R}^{n}.\, I(\mathbf{x}) \Rightarrow H(\mathbf{x})$.
	\item $\Phi_5 \equiv \forall \mathbf{x} \in \mathbb{R}^{n}.\, I(\mathbf{x}) \backslash T(\mathbf{x}) \Rightarrow R(\mathbf{x}) \geq R(\mathrm{step}(\mathbf{x})) - 1$, where $\mathrm{step}(\mathbf{x})$ is defined as above.
\end{compactitem}

\paragraph{Step 3 -- Constraint Solving.} The initial constraint $\Phi_{\textrm{init}}$ is purely existentially quantified over the symbolic template variables. However, the constraint $\Phi_{\strategy}$ as well as $\Phi_1,\dots,\Phi_5$ all contain alternation of an existential quantifier over the symbolic template variables followed by a universal quantifier over the distribution $\mathbf{x}\in\mathbb{R}^n$ over the MDP state probabilities. Our algorithm first translates these constraints into equisatisfiable \emph{purely existentially quantified} constraints, before feeding the resulting constraints to an off-the-shelf solver.

In order to remove quantifier alternation, we first observe that $\Phi_{\strategy},\Phi_1,...,\Phi_5$ can all be expressed as conjunctions of finitely many constraints of the form
\begin{eqnarray}\label{eq:hornpoly}
	\begin{split}
		\forall\mathbf{x}\in\mathbb{R}^n.\ &(\textrm{affexp}_1(\mathbf{x}) \bowtie 0) \land \dots \land (\textrm{affexp}_N(\mathbf{x}) \bowtie 0) \\
		&\Longrightarrow (\textrm{polyexp}(\mathbf{x}) \bowtie 0),
	\end{split}
\end{eqnarray}
with each $\textrm{affexp}_i(\mathbf{x})$ being an affine expression over $\mathbf{x}$ and $\textrm{poly}(\mathbf{x})$ a polynomial expression over $\mathbf{x}$, whose affine and polynomial coefficients are either concrete real values or symbolic template variables, and each $\bowtie \in \{\geq, >\}$. The reason why we now may have polynomial expressions on the right-hand-side, in contrast to the algorithm in Section~5, is that in $\Phi_2$ and $\Phi_5$ we now have templates for $p_{\state_i,\action_j}(\mathbf{x})$ which are quotients of affine expressions. In order to remove the quotients, the algorithm multiplies them out which gives rise to polynomial expressions over symbolic template variables. This means that we cannot use a Farkas' lemma style result in order to remove quantifier alternation, since Farkas' lemma requires expressions on both sides of the implication to be affine.

To overcome this challenge and remove quantifier alternation, we utilize the following result of~\cite[Corollary~3]{AsadiC0GM21}. This result is a generalization of Handelman's theorem, which can intuitively be viewed as an extension of Farkas' lemma that allows strict polynomial inequalities on the right-hand-side of implications. Hence, the result below provides a recipe to translate a constraint as in eq.~\eqref{eq:hornpoly} into a system of purely existentially quantified constraints.

\begin{lemma}[\cite{AsadiC0GM21}] \label{lemma:handelman}
	Let $\mathcal{X} = \{x_1,\dots,x_n\}$ be a finite set of real-valued variables, and consider the following system of $N\in\mathbb{N}$ affine inequalities over $\mathcal{X}$:
	\begin{equation*}\Phi:\, \begin{cases}
			c^1_0 + c^1_1\cdot x_1 + \dots + c^1_n \cdot x_n \bowtie 0 \\
			\qquad \qquad \qquad \vdots \\
			c^N_0 + c^N_1\cdot x_1 + \dots + c^N_n \cdot x_n \bowtie 0
		\end{cases}.
	\end{equation*}
	with each $\bowtie \in \{\geq,>\}$. Let $\textrm{Prod}(\Phi) = \{ \prod_{i=1}^t \phi_i  \mid t \in \mathbb{N}_0, \phi_i \in \Phi \}$ be the set of all products of finitely many affine expressions in $\Phi$, where the product of $0$ affine expressions is a constant expression $1$.
	Suppose that $\Phi$ is satisfiable and that $\{\mathbf{y} \mid \mathbf{y} \models \Phi\}$, the set of values satisfying $\Phi$, is a bounded set.
	Then $\Phi$ entails a strict polynomial inequality $\phi(\mathbf{x}) > 0$, i.e.\ $\Phi \Longrightarrow \phi(\mathbf{x}) > 0$, if and only if $\phi$ can be written as a non-negative linear combination of finitely many products in $\textrm{Prod}(\Phi)$, i.e.\ if and only if there exist $y_0 > 0$ and $y_1,\dots,y_n\geq 0$ and $\phi_1,\dots,\phi_n\in\textrm{Prod}(\Phi)$ such that $\phi = y_0 + y_1 \cdot \phi_1 + \dots + y_n \cdot \phi_n$.
\end{lemma}

To see that Lemma~\ref{lemma:handelman} is applicable, note that in $\Phi_{\strategy},\Phi_2,\Phi_3,\Phi_4$ we have the predicate $I(\mathbf{x})$ on the left-hand-side, in $\Phi_1$ we have $\Init(\mathbf{x})$ and in $\Phi_5$ we have $I(\mathbf{x})\backslash T(\mathbf{x})$. Since all of these contain $\distribution_0$ and all are constrained to only contain distributions, it follows that the sets of satisfying valuations of these predicates are all non-empty and bounded. Hence, Lemma~\ref{lemma:handelman} can be applied to each constraint. On the other hand, notice that the translation in Lemma~\ref{lemma:handelman} implies strict polynomial inequality on the right-hand-side whereas the constraint in eq.~\eqref{eq:hornpoly} may contain a non-strict polynomial inequality. While this leads to incompleteness, it is sufficient in order to make our translation sound and the translation yields a slightly stricter system of purely existentially quantified constraints.

We now describe how the algorithm applies Lemma~\ref{lemma:handelman} to a constraint of the form as in eq.~\eqref{eq:hornpoly}. First, note that Lemma~\ref{lemma:handelman} does not impose a bound on the number of products of affine expressions that might appear in the translation. Hence, in our translation we introduce an additional algorithm parameter $K$ which is an upper bound on the maximal number of affine expressions appearing in products. Let $\textrm{Prod}_K(\Phi) =  \{\prod_{i=1}^t \phi_i \, \mid \, 0\leq t \leq K,\, \phi_i\in\Phi\}$ be the set of all products of at most $K$ affine expressions, $M_K = |\textrm{Prod}_K(\Phi)|$ be the number of such products and $\textrm{Prod}_K(\Phi) = \{\phi_1,\dots,\phi_{M_K}\}$.

For any constraint of the form as in eq.~\eqref{eq:hornpoly}, we introduce fresh symbolic variables $y_0,y_1,\dots,y_{M_K}$ and translate it into the system of purely existentially quantified constraints
\begin{equation*}
	\begin{split}
		&(y_0 > 0) \land (y_1 \geq 0) \land \dots \land (y_N \geq 0) \land \\
		&(\textrm{polyexp}(\mathbf{x}) \equiv_H y_0 + y_1 \cdot \phi_1(\mathbf{x}) + \dots + y_{M_K} \cdot \phi_{M_K}(\mathbf{x})).
	\end{split}
\end{equation*}
Here, $\textrm{polyexp}(\mathbf{x}) \equiv_H y_0 + y_1 \cdot \phi_1(\mathbf{x}) + \dots + y_{M_K} \cdot \phi_{M_K}(\mathbf{x})$ is a notation for the set of equalities over template variables and $y_0,y_1,\dots,y_{M_K}$ which equate the constant coefficient and coefficients of all monomials over $\{x_1,\dots,x_k\}$ on two sides of the equivalence.

This yields a system of purely existentially quantified constraints over the symbolic template variables as well as fresh symbolic variables introduced by the translation. The algorithm then feeds the resulting purely existentially quantified system to an off-the-shelf SMT solver, and any solution gives rise to a concrete instance of an initial distribution $\distribution_0\in\Init$, a distributionally memoryless policy $\strategy$ and a correct affine distributional reach-avoid certificate $\certificate$ which formally proves distributional reach-avoidance.

The following theorem establishes soundness and an upper bound on the computational complexity of our algorithm.

\begin{theorem*}
	\emph{Soundness}: Suppose that the algorithm returns an initial distribution $\distribution_0\in\Init$, a distributionally memoryless policy $\strategy$ and an affine inductive distributional reach-avoid certificate $\certificate$. Then, the MDP $\MDP$ satisfies existential-$(T,H)$-reach-avoidance under $\strategy$.
	
	\emph{Complexity}: The runtime of the algorithm is in PSPACE in the size of the encoding of the MDP, $\Init$, $T$, $H$ and the template size parameter $N_I \in \mathbb{N}$.
\end{theorem*}

\begin{proof}
	{\em Soundness.}  Suppose that the algorithm returns an initial distribution $\distribution_0\in\Init$, a distributionally memoryless policy $\strategy$ and an affine inductive distributional reach-avoid certificate $\certificate$. This means that Step~3 of the algorithm computed a solution $(\distribution_0, \strategy, R, I, y_0, y_1, \dots, y_N)$ to the system of constraints constructed in Step~3 upon quantifier elimination. By Lemma~\ref{lemma:handelman}, this means that $(\distribution_0, \strategy, R, I)$ form a solution to the system of constraints constructed in Step~2. Since this system encodes that $\distribution_0 \in \Init$, that $\strategy$ is a distributionally memoryless strategy and that $(R,I)$ form a correct distributional reach-avoid certificate, we conclude that theorem follows.

	\smallskip\noindent{\em Complexity.} Steps~1 and~2 as well as the quantifier elimination procedure run in time which is polynomial in the size of the input and the template size $N_I$ (note that the parameter $K$ is assumed to be fixed) and yield a system of constraints which is polynomial in the size of the input. Thus, the query that is given to an SMT-solver is a sentence in the existential first-order theory of the reals of size polynomial in the problem input size, which can be solved in PSPACE. Therefore, the combined computational complexity of the algorithm is in PSPACE.
\end{proof}

\noindent{\bf Extension to unit/universal Policy Synthesis} We now show how the above algorithm can be extended to solve unit and universal Policy Synthesis problems as well. For unit Policy Synthesis, the initial distribution $\mu_0$ is given, hence in Step~1 of the algorithm we do not need to fix a template for it and in Step~2 we remove the constraint~$\Phi_{\textrm{init}}$. For universal Policy Synthesis, we now need distributional reach-avoidance to hold from all initial distributions in $\Init$. Hence, in Step~1 we again do not need to fix a template for $\mu_0$ whereas in Step~2 we set $\Phi_1 \equiv \forall \mathbf{x} \in \mathbb{R}^{n} (\Init(\mathbf{x}) \Rightarrow I(\mathbf{x}))$. In both cases, the rest of the algorithm proceeds analogously as above. Moreover, as in the theorem abpve, both algorithms are sound, relatively complete and of PSPACE runtime in the size of the problem.

\smallskip\noindent{\bf Extension to Policy Verification} The above algorithm is also easily extended to solve the unit/existential/universal Policy Verification problems, assuming that a policy $\pi$ to be verified is distributionally memoryless and conforms to the template considered above. Since in the case of verification problems the policy $\pi$ is given and need not be computed, in Step~1 above we need not fix the templates $p_{s_i,a_j}$ for the policy. We also remove the constraint $\Phi_\pi$ in Step~2. The rest of the algorithms proceed analogously as in the Policy Synthesis case, and analogously as in the theorem above we can show that they are sound, relatively complete and of PSPACE runtime.

\section{Benchmark Details}\label{app:models}

\begin{table}[t]
	\caption{
		Summary of the properties of all considered gridworld models.
		From left to right we list the name, number of states, actions, and transitions, followed by number of initial, goal, limited, and forbidden states.
	} \label{table:model_summary}
	\begin{tabular}{cccccccc}
		       Name         & $|S|$ & Act. & Trans. & I & G & L & F \\
		\midrule
		 \textbf{Running}   &   7   &  19  &   24   & 1 & 1 & 1 & 0 \\
		 \textbf{TwoInit}   &   7   &  18  &   22   & 2 & 1 & 3 & 0 \\
		  \textbf{Double}   &  11   &  30  &   36   & 2 & 2 & 2 & 0 \\
		 \textbf{Slippery}  &  12   &  37  &   48   & 1 & 1 & 3 & 0 \\
		 \textbf{Grid} 5x4  &  15   &  29  &   36   & 1 & 1 & 3 & 0 \\
		 \textbf{Grid} 8x8  &  32   &  99  &  111   & 1 & 1 & 3 & 0 \\
		\textbf{Grid} 20x10 &  88   & 280  &  292   & 2 & 1 & 9 & 4
	\end{tabular}
\end{table}

\begin{table}[t]
	\caption{
		Results on the additional models.
		For each model, we list, from left to right, the time used for invariant generation, the fastest of three solver executions, and the total number of variables, constraints, and operations in the query sent to the SMT solver.
		The first line for each model is the policy verification query, the second line is the policy synthesis query.
		T/O denotes a timeout after 10 minutes.
	} \label{table:results_summary}
	\centering\small\setlength{\tabcolsep}{5pt}
	\begin{tabular}{rccccc}
		            Model & Inv. &  SMT  & Var. & Con. & Ops  \\
		\midrule
		 \textbf{TwoInit} &  1s  & $<$1s &  64  &  81  & 562  \\
		                  &  2s  & $<$1s &  81  & 104  & 759  \\
		\textbf{Slippery} &  3s  & $<$1s &  94  & 121  & 856  \\
		                  & 12s  &  11s  & 130  & 168  & 1318
	\end{tabular}
\end{table}

\begin{figure}[t]
	\centering
	\begin{tikzpicture}
		\foreach [count=\i from 0,evaluate=\i as \x using {mod(\i,3)},evaluate=\i as \y using -floor(\i / 3)] \l in {
			I,X,G,%
			L,T,L,%
			I,X,O%
		}
			\node[box\l] at (\x,\y) {};
	\end{tikzpicture}
	\caption{The model \textbf{TwoInit}.} \label{fig:app:example_start}
\end{figure}

\begin{figure}[t]
	\centering
	\begin{tikzpicture}
		\foreach [count=\i from 0,evaluate=\i as \x using {mod(\i,5)},evaluate=\i as \y using -floor(\i / 5)] \l in {
			I,O,L,O,G,%
			X,s,L,X,X,%
			I,O,R,G,X%
		}
			\node[box\l] at (\x,\y) {};
	\end{tikzpicture}
	\caption{The model \textbf{Double}.}
\end{figure}

\begin{figure}[t]
	\centering
	\begin{tikzpicture}
		\foreach [count=\i from 0,evaluate=\i as \x using {mod(\i,5)},evaluate=\i as \y using -floor(\i / 5)] \l in {
			I,X,t,s,s,%
			O,X,t,s,s,%
			s,s,t,X,G%
		}
			\node[box\l] at (\x,\y) {};
	\end{tikzpicture}
	\caption{The model \textbf{Slippery}.}
\end{figure}

\begin{figure}[t]
	\centering
	\begin{tikzpicture}
		\foreach [count=\i from 0,evaluate=\i as \x using {mod(\i,5)},evaluate=\i as \y using -floor(\i / 5)] \l in {
			I,X,O,O,G,%
			O,s,t,X,O,%
			D,X,X,X,U,%
			R,R,r,R,U%
		}
			\node[box\l] at (\x,\y) {};
	\end{tikzpicture}
	\caption{The model \textbf{Grid} 5x4.} \label{fig:app:example_end}
\end{figure}

\begin{figure}
	\begin{equation*}
		10^{-6} \cdot
		\begin{bmatrix}
			93980 & 2634  & 2564  & 798   & 24   \\
			0     & 20724 & 48298 & 29624 & 1354 \\
			0     & 15531 & 42539 & 39530 & 2400 \\
			0     & 2598  & 10778 & 77854 & 8770 \\
			0     & 0     & 0     & 0     & 10^6
		\end{bmatrix}
	\end{equation*}
	\caption{
		The transition structure of \textbf{Insulin}.
	} \label{fig:app:insulin}
\end{figure}

\begin{figure}
	\begin{equation*}
		\begin{bmatrix}
			\tfrac{1}{80}  & \tfrac{19}{60} & \tfrac{3}{40}   & \tfrac{19}{60} & \tfrac{67}{240} \\
			\tfrac{1}{80}  & \tfrac{1}{20}  & \tfrac{41}{120} & \tfrac{19}{60} & \tfrac{67}{240} \\
			\tfrac{1}{16}  & \tfrac{1}{4}   & \tfrac{3}{8}    & \tfrac{1}{4}   & \tfrac{1}{16}   \\
			\tfrac{1}{80}  & \tfrac{1}{20}  & \tfrac{7}{8}    & \tfrac{1}{20}  & \tfrac{1}{80}   \\
			\tfrac{33}{80} & \tfrac{9}{20}  & \tfrac{3}{40}   & \tfrac{1}{20}  & \tfrac{1}{80}
		\end{bmatrix}
	\end{equation*}
	\caption{
		The transition structure of \textbf{PageRank}.
	} \label{fig:app:pagerank}
\end{figure}

In this section, we provide the details of the additional models used in our evaluation. 

For the smaller models, we provide a pictorial representation in Figures~\ref{fig:app:example_start}-\ref{fig:app:example_end}, for \textbf{TwoInit}, \textbf{Double}, \textbf{Slippery}, and \textbf{Grid} 5x4, respectively. The picture for \textbf{Running} can be found in the main paper. The notation in these pictures is the same as in the main body. That is, the initial distribution is indicated by $I$, where the robots are uniformly distributed among all these cells. From any cell, any robot can move horizontally or vertically to an adjacent cell via non-deterministic moves, as long as the adjoining cell is not marked $X$. Cells marked $X$ are obstacles that cannot be moved to. Cells marked $S$ are stochastic where $10\%$ of the robots remain in the cell while remaining $90\%$ can move to adjoining cell. Similarly, $s$ is a stochastic cell where the move could slip to an orthogonally adjacent cell of the destination. Cells with arrows are strong currents, where each robot has no choice but to move to the next cell.

Cells shaded in orange, called limited cells, are distributional obstacles, i.e.\ at any point only $10\%$ of all robots in the swarm may be in the set of cells shaded orange. This will be part of the safety constraint. Finally, $G$ marks goal cells. The problem is to go from $I$ to $G$ (at least $90\%$ of the swarm must reach $G$), while the dynamics must follow the stochastic constraints, and at the same time avoid obstacles $X$ and satisfy distributional constraints in limited cells. Recall that these induce a \emph{distributional} reach-avoid property, which cannot be expressed in classical logics such as PCTL$^\ast$.

For the larger models, namely \textbf{Grid} 8x8 and 20x10, we do not include a picture to avoid clutter, however defining properties of these can be found in Table~\ref{table:model_summary}. Further, in Table~\ref{table:model_summary}, we give for each model more of the defining properties, including number of states, actions, transitions. This is followed by number of Initial states (denoted I), number of goal states (denoted G), where at least $90\%$ of the swarm must reach. $L$ denoted the number of limited cells (coloured in orange in the Figures), with distributional obstacles. Moreover, \emph{forbidden} states (denoted F in the table) mean that no probability mass at all may enter this region (we use such forbidden states only in the largest Grid example that we tested on).

Further, in Table~\ref{table:results_summary} we also provide additional results regarding our SMT query, for two of the examples. For more details, especially regarding the larger models, we direct the interested reader to the concrete implementation, which generates the models from a readable string input.


For completeness, we also show generated strategies for these models in Figures~\ref{fig:app:example_strat_start}-\ref{fig:app:example_strat_end}.
Note that on each run the resulting strategies likely are different due to randomized initialization of SMT solvers.
We also discuss the overall process in more detail in the following section.

We also provide the transition structure of the two models taken from the literature, \textbf{Insulin} and \textbf{PageRank}, in Figures~\ref{fig:app:insulin} and \ref{fig:app:pagerank}, respectively. Intuitively, \textbf{Insulin} models how the body handles insulin being injected. For \textbf{PageRank}, each state represents a webpage and edges hyperlinks between the pages, while the respective probability models the probability of following these hyperlinks. More details about these models can be found in \cite{DBLP:journals/jacm/AgrawalAGT15,DBLP:conf/qest/ChadhaKVAK11}.

\begin{figure}
	\small
	\begin{verbatim}
		p_0_0: d:10523455/4398046511104
		       s:4398035987649/4398046511104
		p_2_1: d:0 l:0 s:0 u:1
		p_1_1: l:0 r:1 s:0
		p_2_2: s:262143/262144 u:1/262144
		p_0_1: d:0 r:1 s:0 u:0
		p_0_2: s:8191/8192 u:1/8192
	\end{verbatim}
	\caption{
		The strategy synthesized for \textbf{TwoInit}.
	} \label{fig:app:example_strat_start}
\end{figure}

\begin{figure}
	\small
	\begin{verbatim}
		p_0_2: r:1/4 s:3/4
		p_1_2: l:0 r:0 s:0 u:1
		p_2_1: d:1 l:0 s:0 u:0
		p_0_0: r:1 s:0
		p_1_0: d:1/32 l:1/32 r:0 s:15/16
		p_2_0: d:1/2 l:1/2 r:0 s:0
		p_1_1: d:0 r:1/64 s:63/64 u:0
		p_3_0: l:0 r:1 s:0
	\end{verbatim}
	\caption{
		The strategy synthesized for \textbf{Double}.
	}
\end{figure}

\begin{figure}
	\small
	\begin{verbatim}
		p_2_2: l:1 s:0 u:0
		p_4_0: d:1 l:0 s:0
		p_0_1: d:1 s:0 u:0
		p_4_1: d:1/16 l:1/8 s:0 u:13/16
		p_0_0: d:5273/20096 s:14823/20096
		p_1_2: l:1/2 r:0 s:1/2
		p_0_2: r:1/2 s:15/32 u:1/32
		p_2_0: d:0 r:1 s:0
		p_3_1: l:0 r:1 s:0 u:0
		p_2_1: d:0 r:0 s:0 u:1
		p_3_0: d:1 l:0 r:0 s:0
	\end{verbatim}
	\caption{
		The strategy synthesized for \textbf{Slippery}.
	} \label{fig:app:example_strat_end}
\end{figure}

\section{Insight into a Concrete Query}\label{app:smt}

To give the reader a feeling for the structure of the SMT queries, we provide some details on the generated queries for the (small) \textbf{Grid} 5x4 example.
In Figure~\ref{fig:app:grid_constraints}, we show the actual query sent to the SMT solver.
Note that this set of constraints already excludes all $\approx 100$ non-negativity constraints that are added on top.
In Figure~\ref{fig:app:grid_solution}, we then show the solution as reported by the SMT solver.
Finally, Figure~\ref{fig:app:strategy} shows the strategy we derive from that solution.
\begin{figure*}
\scriptsize
\begin{verbatim}
-9*f_027/10 - f_029 + f_032 - f_033*inv_1_1 - f_038 + inv_1_1 = 0
-9*f_047/10 - f_049 + f_052 - f_053*inv_1_1 - f_058 + rank_1 = 0
-9*f_067/10 - f_069 + f_072 - f_073*inv_1_1 - f_078 - 1 = 0
-f_001 + inv_1_1 + inv_1_2 = 0
-f_002 - f_012 - f_013*inv_1_15 + f_018 = 0
-f_003 - f_012 - f_013*inv_1_12 + f_018 = 0
-f_004 - f_012 - f_013*inv_1_16 + f_018 = 0
-f_005 - f_012 - f_013*inv_1_7 + f_018 = 0
-f_006 - f_012 - f_013*inv_1_10 + f_018 - 1 = 0
-f_007 - f_012 - f_013*inv_1_2 + f_018 = 0
-f_008 - f_012 - f_013*inv_1_6 + f_018 = 0
-f_009 + f_012 - f_013*inv_1_1 - f_018 + 1/10 = 0
-f_010 - f_012 - f_013*inv_1_11 + f_018 = 0
-f_011 - f_012 - f_013*inv_1_3 + f_018 = 0
-f_012 - f_013*inv_1_13 + f_018 - f_019 = 0
-f_012 - f_013*inv_1_14 - f_017 + f_018 = 0
-f_012 - f_013*inv_1_4 - f_016 + f_018 = 0
-f_012 - f_013*inv_1_5 + f_018 - f_020 = 0
-f_012 - f_013*inv_1_8 - f_015 + f_018 = 0
-f_012 - f_013*inv_1_9 - f_014 + f_018 - 1 = 0
-f_021 - f_032 - f_033*inv_1_15 + f_038 + inv_1_14 = 0
-f_022 - f_032 - f_033*inv_1_12 + f_038 + inv_1_16 = 0
-f_023 - f_032 - f_033*inv_1_16 + f_038 + inv_1_15 = 0
-f_024 - f_032 - f_033*inv_1_7 + f_038 + inv_1_10 = 0
-f_025 - f_032 - f_033*inv_1_10 + f_038 + inv_1_12 = 0
-f_026 - f_032 - f_033*inv_1_2 + f_038 + inv_1_2*p_0_0_s + inv_1_3*p_0_0_d = 0
-f_028 - f_032 - f_033*inv_1_6 + f_038 + inv_1_2*p_1_1_l/20 + 9*inv_1_3*p_1_1_l/10 + inv_1_4*p_1_1_r/20 +
    inv_1_6*p_1_1_s + inv_1_8*p_1_1_l/20 + 19*inv_1_9*p_1_1_r/20 = 0
-f_030 - f_032 - f_033*inv_1_11 + f_038 + inv_1_11*p_3_0_s + inv_1_13*p_3_0_r + inv_1_8*p_3_0_l = 0
-f_031 - f_032 - f_033*inv_1_3 + f_038 + inv_1_2*p_0_1_u + inv_1_3*p_0_1_s + inv_1_4*p_0_1_d +
    inv_1_6*p_0_1_r = 0
-f_032 - f_033*inv_1_14 - f_037 + f_038 + inv_1_13*p_4_1_u + inv_1_14*p_4_1_s + inv_1_15*p_4_1_d = 0
-f_032 - f_033*inv_1_4 - f_036 + f_038 + inv_1_5 = 0
-f_032 - f_033*inv_1_5 + f_038 - f_040 + inv_1_7 = 0
-f_032 - f_033*inv_1_8 - f_035 + f_038 + inv_1_11*p_2_0_r + inv_1_8*p_2_0_s + inv_1_9*p_2_0_d = 0
-f_032 - f_033*inv_1_9 - f_034 + f_038 + inv_1_3*p_2_1_l/20 + inv_1_3*p_2_1_u/20 + inv_1_5*p_2_1_u/20 +
    9*inv_1_6*p_2_1_l/10 + 9*inv_1_8*p_2_1_u/10 + inv_1_9*p_2_1_l/20 + inv_1_9*p_2_1_s = 0
-f_041 - f_052 - f_053*inv_1_15 + f_058 + rank_15 = 0
-f_042 - f_052 - f_053*inv_1_12 + f_058 + rank_12 = 0
-f_043 - f_052 - f_053*inv_1_16 + f_058 + rank_16 = 0
-f_044 - f_052 - f_053*inv_1_7 + f_058 + rank_7 = 0
-f_045 - f_052 - f_053*inv_1_10 + f_058 + rank_10 = 0
-f_046 - f_052 - f_053*inv_1_2 + f_058 + rank_2 = 0
-f_048 - f_052 - f_053*inv_1_6 + f_058 + rank_6 = 0
-f_050 - f_052 - f_053*inv_1_11 + f_058 + rank_11 = 0
-f_051 - f_052 - f_053*inv_1_3 + f_058 + rank_3 = 0
-f_052 - f_053*inv_1_14 - f_057 + f_058 + rank_14 = 0
-f_052 - f_053*inv_1_4 - f_056 + f_058 + rank_4 = 0
-f_052 - f_053*inv_1_5 + f_058 - f_060 + rank_5 = 0
-f_052 - f_053*inv_1_8 - f_055 + f_058 + rank_8 = 0
-f_052 - f_053*inv_1_9 - f_054 + f_058 + rank_9 = 0
-f_061 - f_072 - f_073*inv_1_15 + f_078 - rank_14 + rank_15 = 0
-f_062 - f_072 - f_073*inv_1_12 + f_078 + rank_12 - rank_16 = 0
-f_063 - f_072 - f_073*inv_1_16 + f_078 - rank_15 + rank_16 = 0
-f_064 - f_072 - f_073*inv_1_7 + f_078 - rank_10 + rank_7 = 0
-f_065 - f_072 - f_073*inv_1_10 + f_078 + rank_10 - rank_12 = 0
-f_066 - f_072 - f_073*inv_1_2 + f_078 - p_0_0_d*rank_3 - p_0_0_s*rank_2 + rank_2 = 0
-f_068 - f_072 - f_073*inv_1_6 + f_078 - p_1_1_l*rank_2/20 - 9*p_1_1_l*rank_3/10 - p_1_1_l*rank_8/20 -
    p_1_1_r*rank_4/20 - 19*p_1_1_r*rank_9/20 - p_1_1_s*rank_6 + rank_6 = 0
-f_070 - f_072 - f_073*inv_1_11 + f_078 - p_3_0_l*rank_8 - p_3_0_r*rank_13 - p_3_0_s*rank_11 + rank_11 = 0
-f_071 - f_072 - f_073*inv_1_3 + f_078 - p_0_1_d*rank_4 - p_0_1_r*rank_6 - p_0_1_s*rank_3 - 
    p_0_1_u*rank_2 + rank_3 = 0
-f_072 - f_073*inv_1_14 - f_077 + f_078 - p_4_1_d*rank_15 - p_4_1_s*rank_14 - p_4_1_u*rank_13 + rank_14 = 0
-f_072 - f_073*inv_1_4 - f_076 + f_078 + rank_4 - rank_5 = 0
-f_072 - f_073*inv_1_5 + f_078 - f_080 + rank_5 - rank_7 = 0
-f_072 - f_073*inv_1_8 - f_075 + f_078 - p_2_0_d*rank_9 - p_2_0_r*rank_11 - p_2_0_s*rank_8 + rank_8 = 0
-f_072 - f_073*inv_1_9 - f_074 + f_078 - p_2_1_l*rank_3/20 - 9*p_2_1_l*rank_6/10 - p_2_1_l*rank_9/20 -
    p_2_1_s*rank_9 - p_2_1_u*rank_3/20 - p_2_1_u*rank_5/20 - 9*p_2_1_u*rank_8/10 + rank_9 = 0
f_027 - f_032 - f_033*inv_1_13 + f_038 - f_039 + inv_1_13 = 0
f_047 - f_052 - f_053*inv_1_13 + f_058 - f_059 + rank_13 = 0
f_067 - f_072 - f_073*inv_1_13 + f_078 - f_079 = 0
p_0_0_d + p_0_0_s = 1
p_0_1_d + p_0_1_r + p_0_1_s + p_0_1_u = 1
p_1_1_l + p_1_1_r + p_1_1_s = 1
p_2_0_d + p_2_0_r + p_2_0_s = 1
p_2_1_l + p_2_1_s + p_2_1_u = 1
p_3_0_l + p_3_0_r + p_3_0_s = 1
p_4_1_d + p_4_1_s + p_4_1_u = 1
\end{verbatim}
\caption{
	Non-trivial constraints generated for the \textbf{Grid} 5x4 example.
}\label{fig:app:grid_constraints}
\end{figure*}

\begin{figure*}
\scriptsize
\begin{minipage}{0.5\textwidth}
\begin{verbatim}
f_001: 0
f_002: 3/80
f_003: 1300498740643/80
f_004: 3/80
f_005: 1300498740803/80
f_006: 1300498740723/80
f_007: 1/10
f_008: 11/160
f_009: 0
f_010: 0
f_011: 181/2560
f_012: 397/80
f_013: 2
f_014: 1363/80
f_015: 0
f_016: 1461998188643/80
f_017: 3/80
f_018: 2637
f_019: 0
f_020: 1300498740803/80
f_021: 0
f_022: 8128117129
f_023: 0
f_024: 0
f_025: 1
f_026: 27775/2097152
f_027: 0
f_028: 80749710969/14073748835532800
f_029: 0
f_030: 0
f_031: 90095/121634816
f_032: 4603
f_033: 1
f_034: 2373649911403/6871947673600
f_035: 0
f_036: 1009371549
f_037: 3/160
f_038: 4603
f_039: 0
f_040: 0
f_041: 249/16
f_042: 764043012923/64
f_043: 473/16
f_044: 373893390317/32
f_045: 377957448429/32
f_046: 63869918880030172383/2048
f_047: 15
f_048: 127739837760060225887/4096
f_049: 1
f_050: 63869918746453404743/2048
f_051: 2043837404160964531791/65536
f_052: 0
f_053: 95/64
f_054: 1995934962913802811/64
f_055: 63869918746453435463/2048
f_056: 1995935833058349597/64
f_057: 25/16
f_058: 31186483828141533
f_059: 3143/2048
f_060: 369829332205/32
f_061: 7/16
f_062: 37/64
f_063: 7/16
f_064: 19/32
f_065: 51/32
f_066: 0
f_067: 55/4
f_068: 564738462784209000817/175921860444160
f_069: 383/2048
f_070: 319349593732267022717/10240
f_071: 0
f_072: 11
f_073: 1/64
f_074: 1/1342177280
f_075: 14717/10240
f_076: 1995935006646875843/64
f_077: 23/16
f_078: 18
f_079: 1917/10240
f_080: 19/32
\end{verbatim}
\end{minipage}
\begin{minipage}{0.5\textwidth}
\begin{verbatim}
inv_1_1: -42111/32
inv_1_10: -8128115814
inv_1_11: 210563/160
inv_1_12: -8128115813
inv_1_13: 210563/160
inv_1_14: 1316
inv_1_15: 1316
inv_1_16: 1316
inv_1_2: 42111/32
inv_1_3: 1347567/1024
inv_1_4: -9137487363
inv_1_5: -8128115814
inv_1_6: 84223/64
inv_1_7: -8128115814
inv_1_8: 210563/160
inv_1_9: 1307
p_0_0_d: 5555/6144
p_0_0_s: 589/6144
p_0_1_d: 0
p_0_1_r: 1871855/1900544
p_0_1_s: 0
p_0_1_u: 28689/1900544
p_1_1_l: 63/64
p_1_1_r: 1/8796093022208
p_1_1_s: 137438953471/8796093022208
p_2_0_d: 0
p_2_0_r: 1
p_2_0_s: 0
p_2_1_l: 3/4
p_2_1_s: 16777215/67108864
p_2_1_u: 1/67108864
p_3_0_l: 0
p_3_0_r: 1
p_3_0_s: 0
p_4_1_d: 0
p_4_1_s: 0
p_4_1_u: 1
rank_1: 63869918880029888735/2048
rank_10: -31186484082143181
rank_11: -65222878
rank_12: -31186483955141366
rank_13: -31186483828139593
rank_14: -31186483828139578
rank_15: -31186483828139564
rank_16: -31186483828139550
rank_2: 153
rank_3: 138
rank_4: -62625
rank_5: -31186484336146813
rank_6: 124
rank_7: -31186484209144997
rank_8: -65222863
rank_9: -32611424
\end{verbatim}
\end{minipage}
\caption{
	One solution the SMT solver found for the query in Figure~\ref{fig:app:grid_constraints}.
} \label{fig:app:grid_solution}
\end{figure*}

\begin{figure*}
\scriptsize
\begin{verbatim}
p_1_1: l:63/64 r:1/8796093022208 s:137438953471/8796093022208
p_0_1: d:0 r:1871855/1900544 s:0 u:28689/1900544
p_2_1: l:3/4 s:16777215/67108864 u:1/67108864
p_4_1: d:0 s:0 u:1
p_0_0: d:5555/6144 s:589/6144
p_2_0: d:0 r:1 s:0
p_3_0: l:0 r:1 s:0
\end{verbatim}
\caption{
	The strategy derived from the SMT solution in Figure~\ref{fig:app:grid_solution}.
} \label{fig:app:strategy}
\end{figure*}

\end{document}